\documentclass[10pt,letterpaper,journal,final,onecolumn]{IEEEtran}
\usepackage{xcolor}

\ifCLASSOPTIONcompsoc
\usepackage[nocompress]{cite}
\else
\usepackage{cite}
\fi

\ifCLASSINFOpdf
\usepackage[pdftex]{graphicx}
\graphicspath{{./fig/}}
\usepackage{epstopdf}
\DeclareGraphicsExtensions{.pdf,.jpeg,.png,.eps}
\else
\usepackage[dvips]{graphicx}
\graphicspath{{./fig/}}
\DeclareGraphicsExtensions{.eps}
\fi

\usepackage[cmex10]{amsmath}
\usepackage{amsthm,amssymb,amsfonts,mathrsfs,bm}
\usepackage{algorithm,algorithmicx,algpseudocode}
\usepackage{ragged2e}
\usepackage{tikz}
\usepackage[inline]{enumitem}
\usepackage{multirow}
\usepackage[us,24hr]{datetime}

\ifCLASSOPTIONcompsoc
\usepackage[caption=false,font=footnotesize,labelfont=sf,
textfont=sf]{subfig} 
\else
\usepackage[caption=false,font=footnotesize]{subfig}
\fi

\newcommand{\squeezeup}{\vspace{-1mm}}

\DeclareMathOperator*{\Argmin}{arg\;min}

\DeclareMathOperator{\Expect}{\mathbb{E}}

\newlength\dlf  

\theoremstyle{definition}

\newtheorem{remark}{Remark}
\newtheorem{theorem}{Theorem}
\newtheorem{lemma}{Lemma}
\newtheorem{definition}{Definition}
\newtheorem{proposition}{Proposition}
\newtheorem{corollary}{Corollary}

\begin{document}
\title{Sketched Subspace Clustering}
%

\author{Panagiotis~A.~Traganitis,~\IEEEmembership{Student
    Member,~IEEE,}
  and~Georgios~B.~Giannakis,~\IEEEmembership{Fellow,~IEEE}
  \IEEEcompsocitemizethanks{\IEEEcompsocthanksitem Panagiotis A. Traganitis and Georgios B. Giannakis are
    with the Dept.\ of Electrical and Computer Engineering and the
    Digital Technology Center, University of Minnesota,
    Minneapolis, MN 55455, USA.\protect\\     \protect\\
    This work was supported by the NSF grants 1500713 and 1514056.\protect\\
    Preliminary parts of this work appeared at the \emph{ Proc. of the $50$th Asilomar Conf. in Signals, Systems, and Computers}, Pacific Grove, CA, Nov. 2016.   
    E-mails: traga003@umn.edu, georgios@umn.edu}
}
\markboth{IEEE Transactions on Signal Processing 2018 (TO APPEAR)}{Traganitis \MakeLowercase{\textit{et al.}}: Sketched Subspace Clustering }

\IEEEtitleabstractindextext{%
  \begin{abstract}\justifying The immense amount of daily generated and communicated data presents unique challenges in their processing. Clustering, the grouping of data without the presence of ground-truth labels, is an important tool for drawing inferences from data.
  	Subspace clustering (SC) is a relatively recent method that is able to successfully classify nonlinearly 
  	separable data in a multitude of settings. In spite of their high clustering accuracy, SC methods incur prohibitively high computational complexity when processing large volumes of high-dimensional data.
    Inspired by random sketching approaches for dimensionality reduction, the present paper
    introduces a randomized scheme for SC, termed Sketch-SC, tailored for large volumes of high-dimensional data. Sketch-SC accelerates the computationally heavy parts of state-of-the-art SC approaches by compressing the data matrix across both dimensions using random projections, thus enabling fast and accurate large-scale SC. Performance analysis as well
    as extensive numerical tests on real data
    corroborate the potential of Sketch-SC and its competitive
    performance relative to state-of-the-art scalable SC approaches.
  \end{abstract}

  }

\maketitle

\IEEEdisplaynontitleabstractindextext


\IEEEpeerreviewmaketitle

\ifCLASSOPTIONcompsoc
\IEEEraisesectionheading{\section{Introduction}}
\else
\section{Introduction}
\label{sec:introduction}
\fi
\IEEEPARstart{T}{he} permeation of the Internet and social networks into our daily life, as well as the ever increasing number of connected devices and highly accurate instruments, has trademarked society and computing research with a ``data deluge''. Naturally, it is desirable to extract information and inferences from the available data. However, the sheer amount of data and their potentially large dimensionality introduces numerous challenges in their processing and analysis, as traditional statistical inference and machine learning processes do not necessarily scale. As the cost of cloud computing is declining, traditional approaches have to be redesigned to take advantage of the flexibility provided by distributed computing across multiple nodes as well as decreasing the computational burden per node, since in many cases each computing node might be an inexpensive machine.

Clustering (a.k.a. unsupervised classification) is a method of grouping data, without having labels available. Also referred to as graph partitioning or community identification, it finds applications in data
mining, signal processing, and machine learning. Arguably, the most popular clustering algorithm is $K$-means due to its
simplicity~\cite{hastie01statisticallearning}. However, $K$-means, as well as
its kernel-based variants, provide meaningful clustering results only
when data, after mapped to an appropriate feature space, form
``tight'' groups that can be separated by
hyperplanes~\cite{hastie01statisticallearning}.

Subspace clustering (SC) on the other hand, is a popular method for clustering
nonlinearly separable data which are generated by a union of (affine)
subspaces in a high-dimensional Euclidean
space~\cite{vidal2010tutorial}. SC has well-documented impact in  applications,
as diverse as image and video segmentation, and identification of
switching linear systems in controls~\cite{vidal2010tutorial}. Recent
advances advocate SC with high clustering performance at
the price of high computational complexity~\cite{vidal2010tutorial}.

The goal of this paper is to introduce a randomized scheme for reducing the computational burden
of SC algorithms when the number of data, and possibly their dimensionality, is prohibitively large, while maintaining high levels of clustering accuracy. Building on random projection (RP) methods, that have been used for dimensionality reduction~\cite{woodruff,boutsidis}, the present paper employs RP matrices to \emph{sketch} and \emph{compress} the available data to a computationally affordable level, while also reducing drastically of optimization variables. In doing so, the proposed method markedly broadens the applicability of high-performing SC algorithms to the big data regime.

Moreover, the present contribution analyzes the performance of Sketch-SC, by leveraging the well-established theory of random matrices and Johnson-Lindenstrauss transforms~\cite{woodruff,JL}. To assess the proposed Sketch-SC scheme, extensive numerical tests on real data are presented, comparing the proposed approach to state-of-the-art SC and large-scale SC methods~\cite{VidalOMP,ORGEN}. Compared to our conference precursor in~\cite{traganitis2016sketchedSC}, comprehensive numerical tests are included here, along with a rigorous performance analysis.

The rest of the paper is organized as
follows. Section~\ref{sec:prelim} provides SC preliminaries along
with notation and prior art.
Section~\ref{sec:proposed_algorithms} introduces the proposed 
Sketch-SC scheme for large-scale datasets, while Section~\ref{sec:performance}
provides pertinent performance bounds. Section~\ref{sec:numerical_tests} presents numerical tests
conducted to evaluate the performance of Sketch-SC in comparison with
state-of-the-art SC and large-scale SC algorithms. Finally, concluding remarks and future
research directions are given in Section~\ref{sec:conclusion}. Proofs of theorems and propositions as well as supporting lemmata are included in Appendix~\ref{app:proofs}.

\noindent\textbf{Notation:} Unless otherwise noted, lowercase bold letters $\bm{x}$ denote
vectors, uppercase bold letters $\mathbf{X}$ represent matrices, and
calligraphic uppercase letters $\mathcal{X}$ stand for sets. The
$(i,j)$th entry of matrix $\mathbf{X}$ is denoted by
$[\mathbf{X}]_{ij}$; $\text{rank}(\mathbf{X})$ and { $\text{range}(\bm{X})$ denote the rank and column span of a matrix $\mathbf{X}$, respectively;} and $\mathbf{X} = \mathbf{U}_\rho\mathbf{\Sigma}_\rho\mathbf{V}_\rho^{\top}$ denotes the singular value decomposition (SVD) of a rank $\rho$, $D\times N$ matrix $\mathbf{X}$, where $\mathbf{U}_\rho$ is $D\times\rho$, $\mathbf{\Sigma}_\rho$ is $\rho\times\rho$, and $\mathbf{V}_\rho$ is $N\times\rho$. For a positive integer $r<\rho$, the SVD of $\mathbf{X}$ can be rewritten as {
\begin{equation}
\begin{aligned}
\mathbf{X} = \mathbf{U}_{\rho}\mathbf{\Sigma}_{\rho}\mathbf{V}_{\rho}^{\top} & = [\mathbf{U}_r \bar{\mathbf{U}}_{r}]
\begin{bmatrix}
\mathbf{\Sigma}_r & \\
& \bar{\mathbf{\Sigma}}_{r}
\end{bmatrix}
\begin{bmatrix}
\mathbf{V}_r^{\top}\\
\bar{\mathbf{V}}_{r}^{\top}
\end{bmatrix} \\
& = \mathbf{X}_r + \bar{\mathbf{X}}_{r}
\end{aligned}
\end{equation}
where $\mathbf{\Sigma}_r$ is an $r\times r$ diagonal matrix with the largest $r$ singular values of $\mathbf{X}$ in descending order, and $\mathbf{X}_r = \mathbf{U}_r\mathbf{\Sigma}_r\mathbf{V}_r^{\top}$ is the best rank-$r$ approximation of $\mathbf{X}$ in the sense that $\mathbf{X}_r$ minimizes $\|\mathbf{X} - \mathbf{X}_r\|_F$. Accordingly, $\bar{\mathbf{\Sigma}}_{r}$ is a $({\rho -r})\times ({\rho -r})$ diagonal matrix containing the remaining singular values of $\mathbf{X}$ and $\bar{\mathbf{X}}_{r} = \bar{\mathbf{U}}_{r}\bar{\mathbf{\Sigma}}_{r}\bar{\mathbf{V}}_{r}^{\top}$.} The
$D$-dimensional real Euclidean space is denoted by $\mathbb{R}^{D}$,  the set of positive real numbers by $\mathbb{R}_{+}$, the set of positive integers by $\mathbb{Z}_{+}$, the
expectation operator by $\Expect[\cdot]$, and the $\ell_2$-norm by $\|\cdot\|$.

\section{Preliminaries}\label{sec:prelim}

\subsection{SC problem statement}\label{ssec:subspaceclustering}

Consider $N$ vectors $\{\bm{x}_{i}\}_{i=1}^{N}$ of size $D\times 1$ drawn from a
union of $K$ affine subspaces, each denoted by
$\mathcal{S}_k$, adhering to the model
\begin{equation}
  \label{eq:pointsubspace}
  \bm{x}_i = \mathbf{C}^{(k)}\bm{y}_{i}^{(k)} + \bm{\mu}^{(k)} + \bm{v}_i\,,\quad
  \forall\bm{x}_i\in\mathcal{S}_k 
\end{equation}
where $d_k$ (possibly with $d_k\ll D$) is the dimensionality of
$\mathcal{S}_k$; {$\mathbf{C}^{(k)}$} is a $D\times d_k$ matrix whose
columns form a basis of $\mathcal{S}_k$; the $d_k$-dimensional vector $\bm{y}_{i}^{(k)}$
is the low-dimensional representation of $\bm{x}_i$ in
$\mathcal{S}_k$ with respect to (w.r.t.) $\mathbf{C}^{(k)}$; the $D\times 1$ vector $\bm{\mu}^{(k)}$ is the ``centroid'' or intercept of
$\mathcal{S}_k$; and, $\bm{v}_i$ denotes the $D\times 1$ noise
vector capturing unmodeled effects. If $\mathcal{S}_k$ is linear, then
$\bm{\mu}^{(k)} = \bm{0}$. 
{
Let also $\bm{\pi}_i$ denote the cluster assignment vector of 
$\bm{x}_i$, and $[\bm{\pi}_{i}]_k$ the $k$th entry of
$\bm{\pi}_i$ that is constrained to satisfy $[\bm{\pi}_{i}]_k\ge 0$ and
$\sum_{k=1}^{K}[\bm{\pi}_{i}]_k = 1$. If $\bm{\pi}_i\in\{0,1\}^{K}$,
then $\bm{x}_i$ lies in only one subspace (hard
clustering), while if $\bm{\pi}_i\in[0,1]^{K}$, then $\bm{x}_i$
can belong to multiple clusters (soft clustering). In the latter case,
$[\bm{\pi}_{i}]_k$ can be thought of as the probability that
$\bm{x}_i$ belongs to $\mathcal{S}_k$. Clearly in the case of hard clustering, \eqref{eq:pointsubspace} can be rewritten as}
\begin{equation}
  \label{eq:allpoints}
  \bm{x}_i = \sum_{k=1}^{K}[\bm{\pi}_{i}]_k\left(\mathbf{C}^{(k)}\bm{y}_{i}^{(k)}
    + \bm{\mu}^{(k)}\right) + \bm{v}_i.
\end{equation}

Given the $D\times N$ data matrix $\mathbf{X} :=
[\bm{x}_1,\bm{x}_2,\ldots,\bm{x}_N]$ and the
number of subspaces $K$, the goal is to find the data-to-subspace
assignment vectors $\{\bm{\pi}_i\}_{i=1}^{N}$, the subspace bases
$\left\{\mathbf{C}^{(k)}\right\}_{k=1}^{K}$, their dimensions
$\{d_k\}_{k=1}^{K}$, the low-dimensional representations
$\{\bm{y}_{i}^{(k)}\}_{i=1}^{N}$, as well as the centroids 
$\{\bm{\mu}^{(k)}\}_{k=1}^{K}$~\cite{vidal2010tutorial}. SC can
be formulated as follows 
\begin{equation} 
  \label{eq:subspaceclustering}
  \begin{aligned}
    &\underset{\mathbf{\Pi},\{\mathbf{C}^{(k)}\},\{\bm{y}_{i}^{(k)}\},\mathbf{M}}{\min} 
    &&\sum_{k=1}^{K}\sum_{i=1}^{N} [\bm{\pi}_{i}]_k\|\bm{x}_i -
    \mathbf{C}^{(k)}\bm{y}_{i}^{(k)} - \bm{\mu}^{(k)} \|_2^2 \\ 
    &\text{subject to (s.to)} && \mathbf{\Pi}^{\top}\bm{1} = \bm{1};
    \quad [\bm{\pi}_{i}]_k\geq 0,\ \forall (i,k)
  \end{aligned}
\end{equation}
where $\mathbf{\Pi} := [\bm{\pi}_1,\ldots,\bm{\pi}_N]$, $\mathbf{M} :=
[\bm{\mu}^{(1)},\bm{\mu}^{(2)},\ldots,\bm{\mu}^{(k)}]$,
and $\bm{1}$ denotes the all-ones vector of matching dimensions.

The problem in \eqref{eq:subspaceclustering} is non-convex as all of
$\mathbf{\Pi}, \{\mathbf{C}^{(k)}\}_{k=1}^{K}, \{d_k\}_{k=1}^{K}, \{\bm{y}_{i}^{(k)}\}$,
and $\mathbf{M}$ are unknown. 
It is known that when $K=1$ and
$\mathbf{C}$ is orthonormal, \eqref{eq:subspaceclustering} boils down
to PCA~\cite{jolliffe2002principal}
\begin{equation} 
  \label{eq:PCAclustering}
  \begin{aligned}
    & \underset{\mathbf{C},\{\bm{y}_{i}\},\bm{\mu}}{\min}
    && \sum_{i=1}^{N} \|\bm{x}_i -
    \mathbf{C}\bm{y}_{i} - \bm{\mu} \|_2^2 \\ 
    &\text{s.to} && \mathbf{C}^{\top} \mathbf{C} = \mathbf{I}
  \end{aligned}
\end{equation}
where $\mathbf{I}$ denotes the identity matrix of appropriate dimension. Notice that for $K=1$, it holds that
$[\bm{\pi}_{i}]_k=1$. Moreover, if
$\mathbf{C}^{(k)} := \bm{0}$, $\forall k$, looking for
$\{\bm{\mu}^{(k)}\}_{k=1}^{K}$, $\{\bm{\pi}_i\}_{i=1}^{N}$ with $K>1$, amounts to $K$-means
clustering
\begin{equation} 
  \label{eq:normalclustering}
  \begin{aligned}
    &\underset{\mathbf{\Pi},\mathbf{M}}{\min}
    &&\sum_{k=1}^{K}\sum_{i=1}^{N} [\bm{\pi}_{i}]_k \|\bm{x}_i -
    \bm{\mu}^{(k)} \|_2^2 \\ 
    &\text{s.to} && \mathbf{\Pi}^{\top}\bm{1} = \bm{1} \,.
  \end{aligned}
\end{equation}
  

\subsection{Prior work}\label{ssec:prior}
Various algorithms have been
developed by the machine learning~\cite{vidal2010tutorial} and
data-mining community~\cite{parsons2004subspace} to solve
\eqref{eq:subspaceclustering}.
Generalizing the ubiquitous $K$-means~\cite{lloydkmeans} the $K$-subspaces algorithm~\cite{ksubspaces} builds on alternating optimization to solve \eqref{eq:subspaceclustering}. For
 $\bm{\Pi}$ and $\{d_k\}_{k=1}^{K}$ fixed, bases of the subspaces can be
 recovered using the SVD on the data associated with each subspace. Indeed, given
 $\mathbf{X}^{(k)} := [\bm{x}_{i_1}, \ldots, \bm{x}_{i_{N_k}}]$, belonging to $\mathcal{S}_k$
 ($\sum_{k=1}^K N_k = N$), a basis $\mathbf{C}^{(k)}$ can be obtained from the first $d_k$ (from
 the left) singular vectors of
 $\mathbf{X}^{(k)} - [\bm{\mu}^{(k)}, \ldots, \bm{\mu}^{(k)}]$, where $\bm{\mu}^{(k)} =
 (1/N_k)\sum_{i\in\mathcal{S}_k}\bm{x}_i$. On the other hand, when
 $\{\mathbf{C}^{(k)}, \bm{\mu}^{(k)}\}_{k=1}^K$ are given, the assignment
 matrix $\bm{\Pi}$ can be recovered in the case of hard clustering by
 finding the closest subspace to each datapoint; that is, $\forall i\in
 \{1,2,\ldots,N\}$, $\forall k\in\{1, \ldots, K\}$, {we obtain
 \begin{equation}
 \label{eq:findclosestsubspace}
 [\bm{\pi}_{i}]_k = \begin{cases}
 1, & \text{if}\ k = \Argmin\limits_{k'\in \{1, \ldots, K\}}
 \left\| \tilde{\bm{x}}_{i}^{(k')} -
 \mathbf{C}^{(k')}{\mathbf{C}^{(k')}}^{\top} \tilde{\bm{x}}_{i}^{(k')}\right\|_2^2\\   
 0, & \text{otherwise}
 \end{cases}  
 \end{equation}
 where $\tilde{\bm{x}}_{i}^{(k)} := \bm{x}_i - \bm{\mu}^{(k)}$ and $\| \tilde{\bm{x}}_{i}^{(k)} - \mathbf{C}^{(k)}{\mathbf{C}^{(k)}}^{\top} \tilde{\bm{x}}_{i}^{(k)}
 \|_2$ is the distance of $\bm{x}_i$ from $\mathcal{S}_k$.} Thus, the $K$-subspaces algorithm operates as follows: (i) Fix $\mathbf{\Pi}$ and
 solve for the remaining unknowns; and (ii) fix
 $\{\mathbf{C}^{(k)}, \bm{\mu}^{(k)}\}_{k=1}^K$, and solve for
 $\mathbf{\Pi}$. Since SVD is involved, SC entails high
 computational complexity, whenever $d_k$ and/or $N_k$ are
 massive. 
 
A probabilistic (soft) counterpart of
$K$-subspaces is the mixture of probabilistic
PCA~\cite{tipping1999mixtures}, which assumes that data are drawn from
a mixture of degenerate (zero-variance) Gaussians. Building on the same
assumption, the agglomerative lossy compression (ALC)
minimizes the required number of bits to ``encode'' each cluster, up to a certain distortion level~\cite{ALC2007}. Algebraic schemes, such as generalized (G)PCA approach SC from a linear
algebra point of view, but generally their performance is guaranteed
only for independent and noise-less subspaces~\cite{gpca}. Additional interesting
methods recover
subspaces by finding local linear subspace approximations~\cite{zhang2012hybrid}; by thresholding the correlations between data~\cite{heckel2015robust}; or by identifying the subspaces one by one~\cite{rahmani2015innovation}. Recently, multilinear methods for SC of tensor data have also been advocated~\cite{traganitis2016parafac}; see also \cite{zhang2009median,traganitis2016online,onlineLRR} for online
clustering approaches to handle streaming data.

Arguably the most successful class of solvers for
\eqref{eq:subspaceclustering} relies on \emph{spectral
clustering}~\cite{spectralclustering} to find the data-to-subspace
assignments. Algorithms in this class generate first an $N\times N$ symmetric weighted adjacency
matrix $\mathbf{W}$ to capture the non-directional
similarity between data vectors, and then perform spectral clustering on
$\mathbf{W}$. Matrix $\mathbf{W}$ implies a graph $\mathcal{G}$ whose
vertices correspond to data and the weight of the edge connecting vertex $i$ and vertex $j$ is given by $[\mathbf{W}]_{ij}$. Spectral
clustering algorithms form the graph Laplacian matrix
\begin{equation}
  \label{eq:laplacian}
  \mathbf{L} := \text{diag}(\mathbf{W}\bm{1}) - \mathbf{W}
\end{equation}
where $\text{diag}(\mathbf{W}\bm{1})$ is a diagonal matrix holding $\mathbf{W}\bm{1}$ on its diagonal. The algebraic multiplicity of the $0$ eigenvalue
of $\mathbf{L}$ yields the number of connected components in
$\mathcal{G}$, while the corresponding eigenvectors are indicator vectors of these connected
components~\cite{spectralclustering}. Afterwards, having formed $\mathbf{L}$, the $K$ eigenvectors $\{{\bf{v}}_k\}_{k=1}^{K}$ corresponding to the trailing eigenvectors of $\mathbf{L}$ are found, and $K$-means is performed on the rows of the $N\times K$ matrix $\mathbf{V}:=[{\bf{v}}_1,\ldots,{\bf{v}}_K]$ to obtain clustering assignments~\cite{spectralclustering}. 


Sparse subspace clustering (SSC)~\cite{elhamifar2013SSC} exploits the
fact that under the union of subspaces model \eqref{eq:subspaceclustering}, only a small percentage
of data suffices to provide a low-dimensional affine representation of $\bm{x}_i$; that is, $\bm{x}_i = \sum_{j=1,j\neq
  i}^{N}w_{ij}\bm{x}_j$, $\forall i\in\{1,2,\ldots,N\}$. Specifically,
SSC solves the following sparsity-promoting optimization problem
\begin{equation}
  \label{eq:SSC}
  \begin{aligned}
    & \underset{\mathbf{Z}}{\min} && \|\mathbf{Z}\|_1 +
    \frac{\lambda}{2}\|\mathbf{X} - \mathbf{X}\mathbf{Z}\|_F^2 \\ 
    & \text{s.to} && \mathbf{Z}^{\top}\bm{1} = \bm{1};\quad
    \text{diag}(\mathbf{Z}) = \bm{0}
  \end{aligned}
\end{equation}
where $\mathbf{Z} := [\bm{z}_1,\bm{z}_2,\ldots,\bm{z}_N]$; column
$\bm{z}_i$ is sparse and contains the coefficients for the
representation of $\bm{x}_i$; $\lambda>0$ is the regularization
coefficient; and $\|\mathbf{Z}\|_1 :=
\sum_{i,j=1}^{N}|[\mathbf{Z}]_{i,j}|$. The constraint $ \text{diag}(\mathbf{Z}) = \bm{0}$ ensures that the solution of the optimization problem is not a trivial one ($\mathbf{Z} = \mathbf{I}$), while $\mathbf{Z}^{\top}\bm{1} = \bm{1}$ is employed to guarantee that the $\mathbf{Z}$ found is invariant to shifting the data by a constant vector~\cite{vidal2010tutorial}. Matrix $\mathbf{Z}$ is used to
create the {weighted adjacency} matrix $[\mathbf{W}]_{ij} := |[\mathbf{Z}]_{ij}| +
|[\mathbf{Z}]_{ji}|$. Finally, spectral clustering, is performed on $\mathbf{W}$ and cluster
assignments are identified. Using those assignments, $\mathbf{M}$ is
found by taking sample means per cluster, and
$\{\mathbf{C}^{(k)}\}_{k=1}^{K}$, $\{\bm{y}_{i}^{(k)}\}_{i=1}^{N}$ are obtained by
applying SVD on $\mathbf{X}^{(k)} - [\bm{\mu}^{(k)}, \ldots, \bm{\mu}^{(k)}]$. 

The low-rank representation (LRR) approach to SC is similar to
SSC, but replaces the $\ell_1$-norm in \eqref{eq:SSC} with the nuclear
one: $\|\mathbf{Z}\|_* := \sum_{i=1}^{\rho}\sigma_i(\mathbf{Z})$,
where $\rho$ stands for the rank and $\sigma_i(\mathbf{Z})$ for the
$i$th singular value of $\mathbf{Z}$. Specifically, LRR solves the following optimization problem~\cite{LRR}
\begin{equation}
\label{eq:LRR}
\begin{aligned}
& \underset{\mathbf{Z}}{\min} && \|\mathbf{Z}\|_* +
\frac{\lambda}{2}\|\mathbf{X} - \mathbf{X}\mathbf{Z}\|_{2,1} 
\end{aligned}
\end{equation}
{where $\|\mathbf{X}\|_{2,1} := \sum_{j=1}^{N}\|\bm{x}_j\|_2$, and $\bm{x}_j$ denotes the $j$-th column of $\mathbf{X}$.}

Another popular algorithm is termed least-squares regression (LSR)~\cite{LSR}. It solves an optimization problem similar to \eqref{eq:LRR}, but replaces the $\ell_1$/nuclear norm with the Frobenius one. Specifically, LSR solves
\begin{equation}
\label{eq:LSR}
\begin{aligned}
& \underset{\mathbf{Z}}{\min} && \frac{1}{2}\|\mathbf{Z}\|_F^2 +
\frac{\lambda}{2}\|\mathbf{X} - \mathbf{X}\mathbf{Z}\|_F^2
\end{aligned}
\end{equation}
which admits the following closed-form solution $\mathbf{Z}^* = \lambda\left(\lambda\mathbf{X}^{\top}\mathbf{X} + \mathbf{I}\right)^{-1}\mathbf{X}^{\top}\mathbf{X}$.
Combining SSC with LSR, the elastic net SC (EnSC) approaches employ a convex combination of $\ell_1$- and Frobenius-norm regularizers~\cite{panagakis2014elastic,fang2015ensc}.
 The high clustering accuracy
achieved by these self-dictionary methods comes at the price of high
complexity. Solving \eqref{eq:SSC},~\eqref{eq:LRR} or \eqref{eq:LSR} scales {cubically} with the
number of data $N$, on top of performing spectral clustering across
$K$ clusters, which renders these methods computationally prohibitive for
large-scale SC. When data are high-dimensional ($D\gg$), methods based
on (statistical) leverage scores, random projections~\cite{boutsidis,heckel2017dimensionality,pimentel2016necessary,wang2016theoretical}, preconditioning and sampling~\cite{pourkamali2017preconditioned},
or our recent sketching and validation (SkeVa)~\cite{skeva}
approach can be employed to reduce complexity to an affordable level. Random projection based methods left multiply the data matrix $\mathbf{X}$, with a $d\times D$ data-agnostic random matrix, thereby reducing the dimensionality of the data vectors from $D$ to $d$. This type of dimensionality reduction has been shown to reduce computational costs while not incurring significant clustering performance degradation when $d=\mathcal{O}(\sum_{k=1}^{K} d_k)$~\cite{heckel2017dimensionality}.  When the number of data vectors is
large ($N\gg$), the scalable SSC/LRR/LSR approach~\cite{SSSC2} involves drawing randomly
$n<N$ data, performing SSC/LRR/LSR on them, and expressing the rest of the
data according to the clusters identified by that random draw of
samples. While this approach clearly reduces complexity, performance can
potentially suffer as the random sample may not be representative of
the entire dataset, especially when $n\ll N$ and clusters are
unequally populated. Other approaches focus on greedy methods, such as orthogonal matching pursuit (OMP), for solving~\eqref{eq:SSC}~\cite{dyerOMP,VidalOMP}. More recently, an active set method, termed Oracle guided Elastic Net (ORGEN)~\cite{ORGEN}, can be used to reduce the complexity of SSC and EnSC tasks, by solving only for the entries of $\mathbf{Z}$ that correspond to data vectors that are highly correlated.

The present paper
introduces a novel approach based on random projections that creates a compact yet expressive dictionary that can be employed by SSC/LRR/LSR to reduce the number of optimization variables to $\mathcal{O}(nN)$ for $n<N$,
thus yielding low computational complexity. In addition, the proposed approach can be combined with random projection methods to reduce data dimensionality, which further scales down computational costs.

\section{Sketched Subspace Clustering}
\label{sec:proposed_algorithms}
Consider the following unifying optimization problem
{
\begin{equation}
\label{eq:general_sc_eq}
\min_{\mathbf{A}\in\mathcal{C}} h(\mathbf{A}) + \lambda L(\mathbf{X} - \mathbf{B}\mathbf{A})
\end{equation}
where $\mathbf{B}$ is an appropriate $D\times n$ {known} basis matrix (dictionary), $h(\mathbf{A})$ is a regularization function of the $n\times N$ matrix $\mathbf{A}$, $L(\cdot)$ is an appropriate loss function, and $\mathcal{C}$ is a constraint set for $\mathbf{A}$. Eq.~\eqref{eq:general_sc_eq} will henceforth be referred to as \emph{Sketch-SC objective}. As mentioned in Sec.\ref{ssec:prior}, the ability of $\mathbf{A}$, obtained from \eqref{eq:general_sc_eq} to distinguish data for clustering depends on the choice of $h(\cdot)$, and on $\mathbf{B}$. For SSC, LSR and LRR, $\mathbf{B} = \mathbf{X}$, $n=N$ and $h(\cdot)$ is $\|\cdot\|_1,\frac{1}{2}\|\cdot\|_F^2$,  $\|\cdot\|_*$, and $L(\cdot)$ is $\frac{1}{2}\|\cdot\|_F^2, \frac{1}{2}\|\cdot\|_F^2$ and $\frac{1}{2}\|\cdot\|_F^2$ or $\frac{1}{2}\|\cdot\|_{2,1}$ respectively.} The constraint set for SSC is $\mathcal{C} = \{\mathbf{A}\in\mathbb{R}^{N\times N}:\mathbf{A}^{\top}\bm{1} = \bm{1}; \text{diag}(\mathbf{A}) = \bm{0}\}$, while for LSR and LRR, we have $\mathcal{C} = \mathbb{R}^{N\times N}$.
\subsection{High volume of data}
\label{ssec:high_vol}
\begin{algorithm}[tb]
	\begin{algorithmic}[1]
		\algrenewcommand\algorithmicindent{1em}
		\Require{$D\times N$ data matrix $\mathbf{X}$; Number of columns of $\mathbf{R}$ $n$; regularization parameter $\lambda$;}
		\Ensure{Model matrix $\mathbf{A}$;}
		\State Generate $N\times n$ JLT matrix $\mathbf{R}$.
		\State Form $D\times n$ dictionary $\mathbf{B} = \mathbf{X}\mathbf{R}$.
		\State Solve \eqref{eq:general_sc_eq} for the cost in \eqref{eq:sketched_LSR},~\eqref{eq:sketched_SSC},~\eqref{eq:sketched_LRR} to obtain $\mathbf{A}$.
	\end{algorithmic}
	\caption{Linear sketched data model for Sketch-SC}\label{alg}
\end{algorithm}

As the aim of the present manuscript is to introduce scalable methods for subspace clustering, the dictionaries considered from now on will have $n\ll N$, bringing the number of variables to $\mathcal{O}(nN)$.
In particular, the dictionaries employed will have the form, $\mathbf{B}:=\mathbf{X}\mathbf{R}$, where $\mathbf{R}$ is a $N\times n$ sketching matrix. The role of $\mathbf{R}$ is to ``compress'' $\mathbf{X}$, while retaining as much information from it as possible. To this end, the celebrated Johnson-Lindenstrauss lemma~\cite{JL} will be invoked. 
\begin{lemma}~\cite{JL}
	\emph{
	\label{lemma:JL}
	Given $\varepsilon>0$, for any subset $\mathcal{V}\subset\mathbb{R}^N$ containing $d$ vectors of size $N\times 1$, there exists a map $q:\mathbb{R}^{N}\rightarrow\mathbb{R}^{n}$ such that for $n\geq n_0 = \mathcal{O}(\varepsilon^{-2}\log{d})$, it holds for all $\bm{x},\bm{y}\in\mathcal{V}$
	\begin{equation}
	(1-\varepsilon)\|\bm{x}-\bm{y}\|_2^2 \leq \|q(\bm{x}) - q(\bm{y})\|_2^2 \leq (1+\varepsilon)\|\bm{x} - \bm{y}\|_2^2.
	\end{equation}
	}
\end{lemma}
In particular, random matrices known as Johnson-Lindenstrauss transforms will be employed since they exhibit useful properties.
\begin{definition}~\cite[Def. 2.3]{woodruff},\cite{boutsidis}
	\label{def:JLT}
	\emph{
	An $N\times n$ random matrix $\mathbf{R}$ forms a Johnson-Lindenstrauss transform (JLT($\varepsilon,\delta,d$)) with parameters $\varepsilon,\delta,d$  if there exists a function $f$, such that for any $\varepsilon>0,\delta < 1, d\in\mathbb{Z}_{+}$ and $d$-element subset $\mathcal{V}\subset\mathbb{R}^{N}$, with $n=\Omega(\frac{\log d}{\varepsilon^2}f(\delta))$, it holds that
	\begin{equation*}
		{\rm Pr}\left\{(1-\varepsilon)\|\bm{x}\|_2^2 \leq \|\bm{x}^{\top}\mathbf{R}\|_2^2\leq (1+\varepsilon)\|\bm{x}\|_2^2\right\} \geq 1 - \delta
	\end{equation*}
	for any $1\times N$ vector $\bm{x}^{\top}\in \mathcal{V}$.}
\end{definition}

One example of a random JLT matrix is a matrix with independent and identically distributed (i.i.d.) entries drawn from a normal $\mathcal{N}(0,1)$ distribution scaled by a factor $1/\sqrt{n}$~\cite{woodruff}. Rescaled random sign matrices, that is matrices with i.i.d. $\pm1$ entries multiplied by $1/\sqrt{n}$ are also  JLTs~\cite{achlioptas,boutsidis}, and matrix products involving these matrices can be computed fast~\cite{Mailman}. Another class of JLTs that allows for efficient matrix multiplication includes the so-called Fast (F)JLTs. This class of FJLTs samples randomly and rescales rows of a fixed orthonormal matrix, such as the discrete Fourier transform (DFT) matrix, or, the Hadamard matrix~\cite{ailon2009fast,ailon2009fast2}; see also~\cite{woodruff,anaraki2014memory,clarkson2013low} where sparse JLT matrices have been advocated.

The following proposition proved in the appendix justifies the use of JLTs for constructing our dictionary $\mathbf{B}$ in \eqref{eq:general_sc_eq}.
\begin{proposition}
	\label{prop:full_range}
	\emph{
		Let $\mathbf{X}$ be a $D\times N$ matrix such that $\text{rank}(\mathbf{X}) = \rho$, and define the $D\times n$ matrix $\mathbf{B}:=\mathbf{X}\mathbf{R}$, where $\mathbf{R}$ is a JLT($\varepsilon,\delta,D$) of size $N\times n$. If $n=\mathcal{O}(\rho\frac{\log(\rho/\varepsilon)}{\varepsilon^2}f(\delta))$ then w.p. at least $1-\delta$, it holds that 
		\begin{equation*}
		\text{range}(\mathbf{X}) = \text{range}(\mathbf{B}).
		\end{equation*}}
\end{proposition}
This proposition asserts that with a proper choice of the sketching matrix $\mathbf{R}$, the dictionary $\mathbf{B}$ is as expressive as $\mathbf{X}$ for solving \eqref{eq:general_sc_eq}, as it preserves the column space of $\mathbf{X}$ with high probability. The next proposition provides a similar bound on the reduced dimension $n$, when $n<\text{rank}(\mathbf{X}):=\rho$.
{
\begin{proposition}
	\label{prop:r_range}
	\emph{
		Let $\mathbf{X}$ be a $D\times N$ matrix such that $\text{rank}(\mathbf{X}) = \rho$, and define the $D\times n$ matrix $\mathbf{B}:=\mathbf{X}\mathbf{R}$, where $\mathbf{R}$ is a JLT($\varepsilon,\delta,D$) of size $N\times n$. If $n=\mathcal{O}(r\frac{\log(r/\varepsilon)}{\varepsilon^2}f(\delta))$, then w.p. at least $1-2\delta$ it holds that
		\begin{equation*}
			\|\mathbf{B}(\mathbf{V}_r^{\top}\mathbf{R})^{\dagger} - \mathbf{U}_r\mathbf{\Sigma}_r\|_F \leq (\varepsilon\frac{\sqrt{1+\varepsilon}}{\sqrt{1-\varepsilon}}+1+\varepsilon)\|\bar{\mathbf{X}}_{r}\|_F.
		\end{equation*}
		}
\end{proposition}
\noindent Prop.~\ref{prop:r_range} suggests that $\mathbf{B}$ approximately inherits the range of $\mathbf{X}_r$.}

Upon constructing a $\mathbf{B}$ adhering to Prop.~\ref{prop:full_range} or Prop.~\ref{prop:r_range}, \eqref{eq:general_sc_eq} can be solved for different choices of $h$. When $h(\mathbf{A}) = \frac{1}{2}\|\mathbf{A}\|_F^2$, the optimization task (termed henceforth \emph{Sketch-LSR})
\begin{equation}
\label{eq:sketched_LSR}
\min_{\mathbf{A}}~ \frac{1}{2}\|\mathbf{A}\|_F^2 + \frac{\lambda}{2}\|\mathbf{X} - \mathbf{B}\mathbf{A}\|_F^2
\end{equation}
is solved by $\mathbf{A}^* = \lambda\left(\lambda\mathbf{B}^{\top}\mathbf{B} + \mathbf{I}\right)^{-1}\mathbf{B}^{\top}\mathbf{X}$, incurring complexity $\mathcal{O}(n^3 + n^2D + nDN)$. Accordingly, our \emph{Sketch-SSC} corresponds to $h(\mathbf{A}) = \|\mathbf{A}\|_1 = \sum_{ij}|[\mathbf{A}]_{ij}|$ and relies on the objective
\begin{equation}
\label{eq:sketched_SSC}
\underset{\mathbf{A}}{\min}~ \|\mathbf{A}\|_1 +
\frac{\lambda}{2}\|\mathbf{X} - \mathbf{B}\mathbf{A}\|_F^2 \\ 
\end{equation}
that can be solved efficiently to obtain $\mathbf{A}$ using the alternating direction method of multipliers (ADMM)~\cite{ADMMGG}, as per~\cite{elhamifar2013SSC}, or any other efficient LASSO solver. The ADMM solver for \eqref{eq:sketched_SSC} incurs complexity  $\mathcal{O}(n^3 + n^2D + nDN + n^2NI)$, where $I$ is the required number of iterations until convergence, and the constraint $\text{diag}(\mathbf{A}) = \bm{0}$ is no longer required as $\mathbf{I}$ is not a trivial solution of \eqref{eq:sketched_SSC}. Proceeding along similar lines, our \emph{Sketch-LRR} objective, for $h(\mathbf{A}) = \|\mathbf{A}\|_*$ aims at
\begin{equation}
\label{eq:sketched_LRR}
\min_{\mathbf{A}}~ \|\mathbf{A}\|_* + \frac{\lambda}{2}\|\mathbf{X} - \mathbf{B}\mathbf{A}\|_F^2
\end{equation}
that can be solved using the augmented Lagrange multiplier (ALM) method of~\cite{LRR}, which incurs complexity $\mathcal{O}(n^3 + n^2D + nDN + (nDN + nN^2 + n^2N)I)$, where $I$ is the number of iterations until convergence. In addition, \eqref{eq:sketched_LRR} can be solved using the $\ell_{2,1}$ norm instead of the Frobenius norm for the fitting term $\mathbf{X} - \mathbf{B}\mathbf{A}$. The entire process to obtain the data model $\mathbf{A}$ is outlined in Alg.~\ref{alg}. {Detailed algorithms for solving~\eqref{eq:sketched_SSC} and~\eqref{eq:sketched_LRR} are described in Appendix~\ref{app:alg}.}

{
\begin{remark}
	An optimal data-driven choice of $\mathbf{R}$ would be interesting only if finding it incurs manageable complexity - a topic which goes beyond the scope of this submission and constitutes a worthy future research direction.
\end{remark}
}
\begin{remark}
	Upon computing $\mathbf{B}$, \eqref{eq:sketched_LSR} and \eqref{eq:sketched_SSC} can be readily parallelized across columns of $\mathbf{X}$. In the nuclear norm case of \eqref{eq:sketched_LRR} one can employ the following identity~\cite{OnlineLRRTraganitis,onlineLRR} 
	\begin{equation}
	\label{eq:nuclearnorm_trick}
		\|\mathbf{A}\|_* = \min_{\mathbf{Z} = \mathbf{P}\mathbf{Q}^{\top}}\frac{1}{2}(\|\mathbf{P}\|_F^2 + \|\mathbf{Q}\|_F^2)
	\end{equation}
	where $\mathbf{A}$ is some $n\times N$ matrix of rank $\rho$ and $\mathbf{P}$ and $\mathbf{Q}$ are $n\times\rho$ and $N\times\rho$ matrices respectively. This is especially useful when multiple computing nodes are available, or the data is scattered across multiple devices. Without \eqref{eq:nuclearnorm_trick}, distributed solvers of \eqref{eq:sketched_LRR} are challenged because as columns of $\mathbf{A}$ are added the SVD needed to find the nuclear norm has to be recomputed, which is not the case with~\eqref{eq:nuclearnorm_trick}.
\end{remark}

{
\begin{remark}
	Existing general guidelines for choosing the regularization parameter $\lambda$ for SSC and LRR~\cite{elhamifar2013SSC,LRR} rely on cross-validation and apply also to the proposed Algs.~\ref{alg} and \ref{alg:highD} here.
\end{remark}
}

\subsection{High-dimensional data}
\label{ssec:high_dim}
The complexity of all the aforementioned algorithms depends on the data dimensionality $D$. As such, datasets containing high-dimensional vectors will certainly increase the computational complexity. As mentioned in Sec.~\ref{ssec:prior}, dimensionality reduction techniques can be employed to reduce the computational burden of SC approaches. Using PCA for instance, a $d<D$-dimensional subspace that describes most of the data variance can be found. This, however, can be prohibitively expensive for large-scale datasets where $N\gg$. 
For such cases, our idea is to combine the method described in the previous section with randomized dimensionality reduction techniques~\cite{heckel2017dimensionality}. Let $\check{\mathbf{R}}$ be a $d\times D$ JLT matrix, where $d\ll D$ is the target dimensionality, and consider the $d\times N$ matrix $\check{\mathbf{X}}:=\check{\mathbf{R}}\mathbf{X}$, which is a reduced dimensionality version of the original data $\mathbf{X}$. The Sketch-SC objective then becomes
\begin{equation}
\label{eq:general_dimred}
\min_{\mathbf{A}} h(\mathbf{A}) + \lambda L(\check{\mathbf{X}} - \check{\mathbf{B}}\mathbf{A})
\end{equation}
where $\check{\mathbf{B}}:=\check{\mathbf{X}}\mathbf{R}$ is a $d\times n$ dictionary of reduced dimension with $\mathbf{R}$ being an $N\times n$ JLT matrix as in \eqref{eq:general_sc_eq}. Upon forming $\check{\mathbf{X}}$ and $\check{\mathbf{B}}$, \eqref{eq:general_dimred} can be solved for different choices of $h$ as in Sec.~\ref{ssec:high_vol}. The steps of our algorithm for high-dimensional data are summarized in Alg.~\ref{alg:highD}.

\begin{remark}
	While carrying out the products $\mathbf{X}\mathbf{R}$, $\check{\mathbf{R}}\mathbf{X}$ or $\check{\mathbf{X}}\mathbf{R}$ can be computationally expensive in cases, they can be accelerated using modern numerical linear algebra tools, such as the Mailman algorithm~\cite{Mailman} {or by employing the Welsh-Hadamard transform~\cite{fastfood,pourkamali2017preconditioned}.} 
\end{remark}

\begin{algorithm}[tb]
	\begin{algorithmic}[1]
		\algrenewcommand\algorithmicindent{1em}
		\Require{$D\times N$ data matrix $\mathbf{X}$; Lower dimension $d$; Number of columns of $\mathbf{R}$ $n$; regularization parameter $\lambda$;}
		\Ensure{Model matrix $\mathbf{A}$;}
		\State Generate $d\times D$ JLT matrix $\check{\mathbf{R}}$.
		\State Generate $N\times n$ JLT matrix $\mathbf{R}$.
		\State Form $d\times N$ matrix $\check{\mathbf{X}} = \check{\mathbf{R}}\mathbf{X}$.
		\State Create $d\times n$ dictionary $\check{\mathbf{B}} = \check{\mathbf{X}}\mathbf{R}$.
		\State Solve \eqref{eq:general_dimred} to obtain sketched data model $\mathbf{A}$.
	\end{algorithmic}
	\caption{Linear sketched data model for Sketch-SC and $D\gg$}\label{alg:highD}
\end{algorithm}
\subsection{Obtaining cluster assignments using $\mathbf{A}$}
\label{ssec:knn}
After obtaining the $N\times N$ matrix $\mathbf{Z}$ in \eqref{eq:SSC},~\eqref{eq:LRR} or \eqref{eq:LSR}, a typical post-processing step for SSC, LSR, and LRR, is to perform spectral clustering, using $\mathbf{W}:=|\mathbf{Z}| + |\mathbf{Z}^{\top}|$ as the adjacency matrix. 
This step however, is not possible for the matrix $\mathbf{A}$ obtained from \eqref{eq:sketched_LSR},~\eqref{eq:sketched_SSC} or \eqref{eq:sketched_LRR}, because it has size $n\times N$, with $n<N$.

While $\mathbf{A}$ cannot be directly used for spectral clustering, a $k$-nearest neighbor graph~\cite{hastie01statisticallearning} can be constructed from the columns of $\mathbf{A}$. Let $\bm{a}_i$ denote the $i$-th column of $\mathbf{A}$, and $\mathcal{K}_i$ the set of the $k$ columns of $\mathbf{A}$ that are closest to $\bm{a}_i$, in the Euclidean distance sense. The $N\times N$ adjacency matrix $\mathbf{W}$ can then be constructed with entries
\begin{equation}
\label{eq:knn_matrix_entries}
[\mathbf{W}]_{ij} = \begin{cases}
1, \quad \text{ if } \bm{a}_j\in\mathcal{K}_i \text{ or } \bm{a}_i\in\mathcal{K}_j \\
0, \quad \text{ otherwise. }
\end{cases}
\end{equation}
In addition, non-binary edge weights can be assigned as
\begin{equation}
\label{eq:knn_matrix_entries_nonbinary}
[\mathbf{W}]_{ij} = \begin{cases}
w_{ij}, \quad \text{  if } \bm{a}_j\in\mathcal{K}_i \text{ or } \bm{a}_i\in\mathcal{K}_j \\
0, \qquad \text{ otherwise. }
\end{cases}
\end{equation}
where $w_{ij}$ is some scalar that depends on $\bm{a}_i$ and $\bm{a}_j$. For instance, if heat kernel weights are used, then $w_{ij} = \exp(-\|\bm{a}_i - \bm{a}_j\|_2^2/\sigma^2)$, for some $\sigma >0$.
The resultant mutual $k$-nearest neighbor matrix $\mathbf{W}$ can then be employed for spectral clustering. Note that the $N\times N$ matrix $\mathbf{W}$ emerging from \eqref{eq:knn_matrix_entries} or \eqref{eq:knn_matrix_entries_nonbinary} will be sparse with $\mathcal{O}(N)$ nonzero entries, which can accelerate the eigendecomposition schemes employed for spectral clustering~\cite{lehoucq1998arpack,kalantzis2016spectral}. The overall scheme is tabulated in Alg.~\ref{alg:knn}.
\begin{remark}
	When $N$ and $n$ are large, computation of the $k$ nearest neighbors can be computationally taxing. Many efficient algorithms are available to accelerate the construction of the $k$ nearest neighbor graph~\cite{l2knng,greedyfiltering}. In addition, approximate nearest neighbor (ANN) methods~\cite{ITQ,indyk1998approximate,slaney2008locality} can be employed to speed up the post-processing step even further. Finally, this post-processing step can be employed for regular SSC, LSR, and LRR.
\end{remark}

\begin{algorithm}[tb]
	\begin{algorithmic}[1]
		\algrenewcommand\algorithmicindent{1em}
		\Require{$n\times N$ matrix $\mathbf{A}$; Number of nearest neighbors $k$; Number of clusters $K$}
		\Ensure{Clustering assignments}
		\State Find $k$-nearest neighbors for each column of $\mathbf{A}$.
		\State Create matrix $\mathbf{W}$ using \eqref{eq:knn_matrix_entries} or \eqref{eq:knn_matrix_entries_nonbinary}.
		\State Apply spectral clustering on $\mathbf{W}$.
	\end{algorithmic}
	\caption{Obtaining clustering assignments from $\mathbf{A}$}\label{alg:knn}
\end{algorithm}
\section{Performance Analysis}
\label{sec:performance} 
In this section, performance of the proposed method will be quantified analytically. Albeit not the tightest, the bounds to be derived will provide nice intuition on why the proposed methods work. The following theorem bounds the representation error of Sketch-LSR in the noise less case.
\begin{theorem} \emph{ Consider noise-free and normalized data vectors obeying \eqref{eq:allpoints} with $\bm{v}_i \equiv \bm{0}$, to form columns of a $D\times N$ data matrix $\mathbf{X}$, with unit $\ell_2$ norm per column, and $\text{rank}(\mathbf{X}) = \rho$. Let also $\mathbf{R}$ denote a JLT($\varepsilon,\delta,D$) of size $N\times n$. Let $\bm{g}^*(\bm{x}) := \mathbf{X}\bm{z}^* = \bm{x}$ denote the representation of $\bm{x}$ provided by LSR, and $\hat{\bm{g}}(\bm{x}) := \mathbf{X}\mathbf{R}\hat{\bm{a}}$ the representation given by Sketch-LSR. If $n=\mathcal{O}(r\frac{\log(r/\varepsilon)}{\varepsilon^2}f(\delta))$, then the following bound holds w.p. at least $1-2\delta$
		\begin{equation*}
		\|\bm{g}^{*}(\bm{x}) - \hat{\bm{g}}(\bm{x})\|_2 \leq \lambda~(1 + \sqrt{\frac{{1+\varepsilon}}{1-\varepsilon}}~\sqrt{\rho - r}~\sigma_{r+1}^2) + \frac{1}{\sqrt{1+\varepsilon}}
		\end{equation*}
		with $\lambda$ as in \eqref{eq:general_sc_eq}, and $\sigma_{r+1}$ denotes the $(r+1)$st singular value of $\mathbf{X}$.}
		\label{thm:LSR}
\end{theorem}

Theorem~\ref{thm:LSR} implies that the larger $n$ is, the smaller the upper bound becomes as a smaller singular value of $\mathbf{X}$ is selected. This also suggests that datasets exhibiting lower rank can be compressed more (with smaller $n$), while retaining representation accuracy.
{ The following corollaries extend the result of Thm.~\ref{thm:LSR} to the Sketch-SSC and Sketch-LRR cases.}
\begin{corollary}
	\emph{
		Consider the setting of Thm.~\ref{thm:LSR}, and let $\hat{\bm{g}}(\bm{x}) := \mathbf{X}\mathbf{R}\hat{\bm{a}}$ be the representation of a datum given by Sketch-SSC.
		The following bound holds w.p. at least $1-2\delta$
		\begin{equation*}
		\|\bm{g}^{*}(\bm{x}) - \hat{\bm{g}}(\bm{x})\|_2 \leq \lambda~(1 + \sqrt{\frac{{1+\varepsilon}}{1-\varepsilon}}~\sqrt{\rho - r}~\sigma_{r+1}^2) + \sqrt{\frac{n}{{1-\varepsilon}}}
		\end{equation*}
		with $\lambda$ as in \eqref{eq:general_sc_eq}, and $\sigma_{r+1}$ denotes the $(r+1)$st singular value of $\mathbf{X}$.}
		\label{corr:ssc}
\end{corollary}
\noindent This corollary is a direct consequence of the fact that for any $n\times 1$ vector $\bm{x}$, it holds that $\|\bm{x}\|_1\leq\sqrt{n}\|\bm{x}\|_2$. Accordingly, the following corollary for Sketch-LRR holds because for any rank $n$ matrix $\mathbf{X}$ we have $\|\mathbf{X}\|_* \leq \sqrt{n}\|\mathbf{X}\|_F$.
\begin{corollary}
		\emph{
			Consider the setting of Thm.~\ref{thm:LSR}, and let ${\bm{g}}^{*}(\mathbf{X}) := \mathbf{X}{\mathbf{Z}}$ and $\hat{\bm{g}}(\mathbf{X}) := \mathbf{X}\mathbf{R}\hat{\mathbf{A}}$ be the representations of all the data given by LRR and Sketch-LRR respectively.
			The following bound holds w.p. at least $1-2\delta$
			{\small\begin{equation*}
			\|\bm{g}^{*}(\mathbf{X}) - \hat{\bm{g}}(\mathbf{X})\|_F \leq \lambda~(\sqrt{N} + \sqrt{\frac{{1+\varepsilon}}{1-\varepsilon}}~\sqrt{\rho - r}~\sigma_{r+1}^2) + \sqrt{\frac{n}{{1-\varepsilon}}}
			\end{equation*}}
			with $\lambda$ as in \eqref{eq:general_sc_eq}, and $\sigma_{r+1}$ denotes the $(r+1)$st singular value of $\mathbf{X}$.
		}
\label{corr:lrr}
\end{corollary}
\noindent For the Sketch-SSC and Sketch-LRR, tighter bounds could possibly be derived by taking into account the special structures of the $\ell_1$ and nuclear norms, instead of invoking norm inequalities.

For a dataset $\mathbf{X}$ drawn from a union of subspaces model, batch methods such as SSC, LSR and LRR, should produce a matrix of representations $\mathbf{Z}$ that is block-diagonal, under certain conditions on the separability of subspaces~\cite{LRR,LSR}. This, in turn, implies that for data $\bm{x}_i,\bm{x}_j\in\mathcal{S}_k,\bm{x}_\ell\in\mathcal{S}_{k'}$ for $k\neq k'$, it holds that
\begin{equation}
\|\bm{z}_i - \bm{z}_j\|_2 \leq \|\bm{z}_i - \bm{z}_\ell\|_2
\end{equation}
that is the representations of two points in the same subspace, are closer than the representations of two points that lie in different subspaces. The following proposition suggests that this property is approximately inherited by the Sketch-SC algorithms of Sec.~\ref{sec:proposed_algorithms}, with high probability.
\begin{proposition}
\label{lemma:dist}
\emph{	Consider $\bm{x}_i= \mathbf{X}\bm{z}_i$ and $\bm{x}_j= \mathbf{X}\bm{z}_j$, and their representation provided by SSC, LRR or LSR $\bm{z}_i$ and $\bm{z}_j$, respectively. Let $\rho = \text{rank}(\mathbf{X})$ and $\bm{a}_i$, $\bm{a}_j$ be the representation obtained by the corresponding Sketch algorithm of Section~\ref{sec:proposed_algorithms}; that is, $\bm{x}_i = \mathbf{X}\mathbf{R}\bm{a}_i$, where the $N\times n$ matrix $\mathbf{R}$ is a JLT($\varepsilon,\delta,D$). If $n=\mathcal{O}(\rho\frac{\log(\rho/\varepsilon)}{\varepsilon^2}f(\delta))$, then w.p. at least $1-\delta$ it holds that
	\begin{equation*}
	\frac{1}{\sqrt{1+\varepsilon}}\|\bm{z}_i - \bm{z}_j\|_2 \leq \|\bm{a}_i - \bm{a}_j\|_2 \leq \frac{1}{\sqrt{1-\varepsilon}}\|\bm{z}_i - \bm{z}_j\|_2.
	\end{equation*}}
\end{proposition}
\noindent Proposition \ref{lemma:dist} also justifies the use of the $k$-nearest neighbor graph as a post-processing step in Sec.~\ref{ssec:knn}.

As will be seen in the ensuing section, the proposed approach has comparable performance to other high-accuracy SC approaches while requiring markedly less time. 

\section{Numerical Tests}\label{sec:numerical_tests}

The proposed method is validated in this section using real
datasets. Sketch-SC methods (termed throught this section as \emph{Sketch-SSC, Sketch-LSR} and \emph{Sketch-LRR}) are compared to SSC, LSR, LRR, the orthogonal matching pursuit method (OMP) for large-scale SC~\cite{VidalOMP}, as well as ORGEN~\cite{ORGEN}. When datasets are large ($N\gg$), the proposed methods are only compared to OMP and ORGEN. The figures of merit
evaluated are following.
%
\begin{itemize}
\item Accuracy, i.e., percentage of correctly clustered data:
\begin{equation*}
\text{Accuracy} := \frac{\text{number of data correctly clustered}}{N}\,.
\end{equation*}

\item Time (in seconds) required for clustering all data. For Algs.~\ref{alg} and \ref{alg:highD} this includes the time required to generate the JLT matrices $\mathbf{R}$, the time required for computing the products $\mathbf{B} = \mathbf{X}\mathbf{R}$, and in the case of Alg.~\ref{alg:highD} $\check{\mathbf{X}}=\check{\mathbf{R}}\mathbf{X}$, $\check{\mathbf{B}} = \check{\mathbf{X}}\mathbf{R}$, as well as the time required for Alg.~\ref{alg:knn}.
\end{itemize} 

All experiments were performed on a machine with an Intel Core-i$5$ $4570$ CPU with $16$GB of RAM. The software used to conduct all experiments is MATLAB~\cite{MATLAB:2015}. $K$-means and ANN were implemented using the VLfeat package~\cite{VLfeat}. All results represent the
averages of $10$ independent Monte Carlo runs. The regularization scalar $\lambda$ [cf. \eqref{eq:SSC}] of SSC and Sketch-SSC is computed as per~\cite[Prop. 1]{elhamifar2013SSC}, and it is controlled by a parameter $\alpha$. ORGEN has two parameters that need to be specified, namely $\lambda$ and $\alpha$. LRR and Sketch-LRR employ the $\ell_{2,1}$ norm for the residual $\mathbf{X} - \mathbf{X}\mathbf{Z}$. For LRR, LSR, Sketch-LRR, Sketch-LSR, OMP and ORGEN the parameters are tuned to optimize empirically the performance of each method considered. 
%
%
%

The real datasets tested are Hopkins 155~\cite{hopkins155}, the
Extended Yale Face dataset~\cite{yaleb}, the COIL-100 database~\cite{COIL100}, and the
MNIST handwritten digits dataset~\cite{MNIST}.
{
\subsection{Assessing the effect of different JLTs}
\label{ssec:simulations_jlts}
Before comparing the proposed scheme with  state-of-the-art competing alternatives, the effect of different JLT matrices on the SC task was tested on two datasets: the Extended Yale Face dataset and the COIL-100 database. The different $N\times n$ JLT matrices assessed are: matrices with i.i.d. $\pm1$ entries rescaled by $1/\sqrt{n}$ (denoted as \emph{Rademacher}); matrices with i.i.d. $\mathcal{N}(0,1)$ entries rescaled by $1/\sqrt{n}$ (denoted as \emph{Normal}); Sparse embedding matrices as described in~\cite{clarkson2013low,woodruff} (denoted as \emph{Sparse}); Fast JLTs using the Hadamard matrix as described in ~\cite{ailon2009fast} (denoted as \emph{Hadamard FJLT}). Fig.~\ref{fig:JLT_test} depicts the performance of Alg.~\ref{alg} for different choices of JLT for the two aforementioned datasets. All JLT matrices achieve comparable performance for the Yale Face database. However, this is not true for the COIL-100 dataset, where the Rademacher JLT seem to provide the most consistent performance.

For all tests in the rest of this section  Algs.~\ref{alg} and \ref{alg:highD} use random matrices $\mathbf{R}$, and $\check{\mathbf{R}}$ that are generated having i.i.d. $\pm1$ entries rescaled by $1/\sqrt{n}$.}
\begin{figure}[tb]
	\centering
	\begin{subfloat}[Extended Yale Face Database]{\includegraphics[width=0.5\columnwidth]{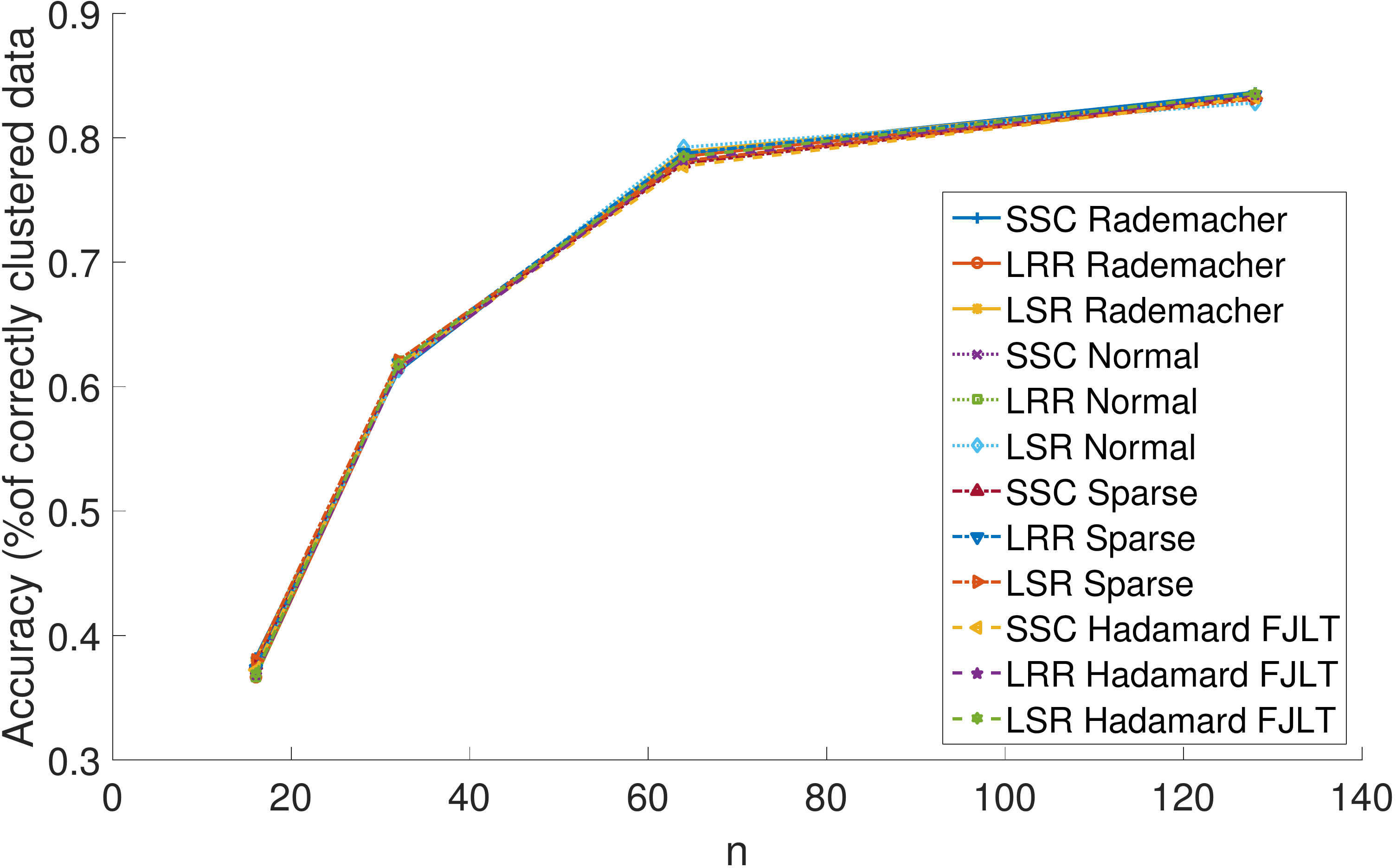} 
			\label{fig:JLT_Yale}}
	\end{subfloat}
	
	\begin{subfloat}[COIL-100]{\includegraphics[width=0.5\columnwidth]{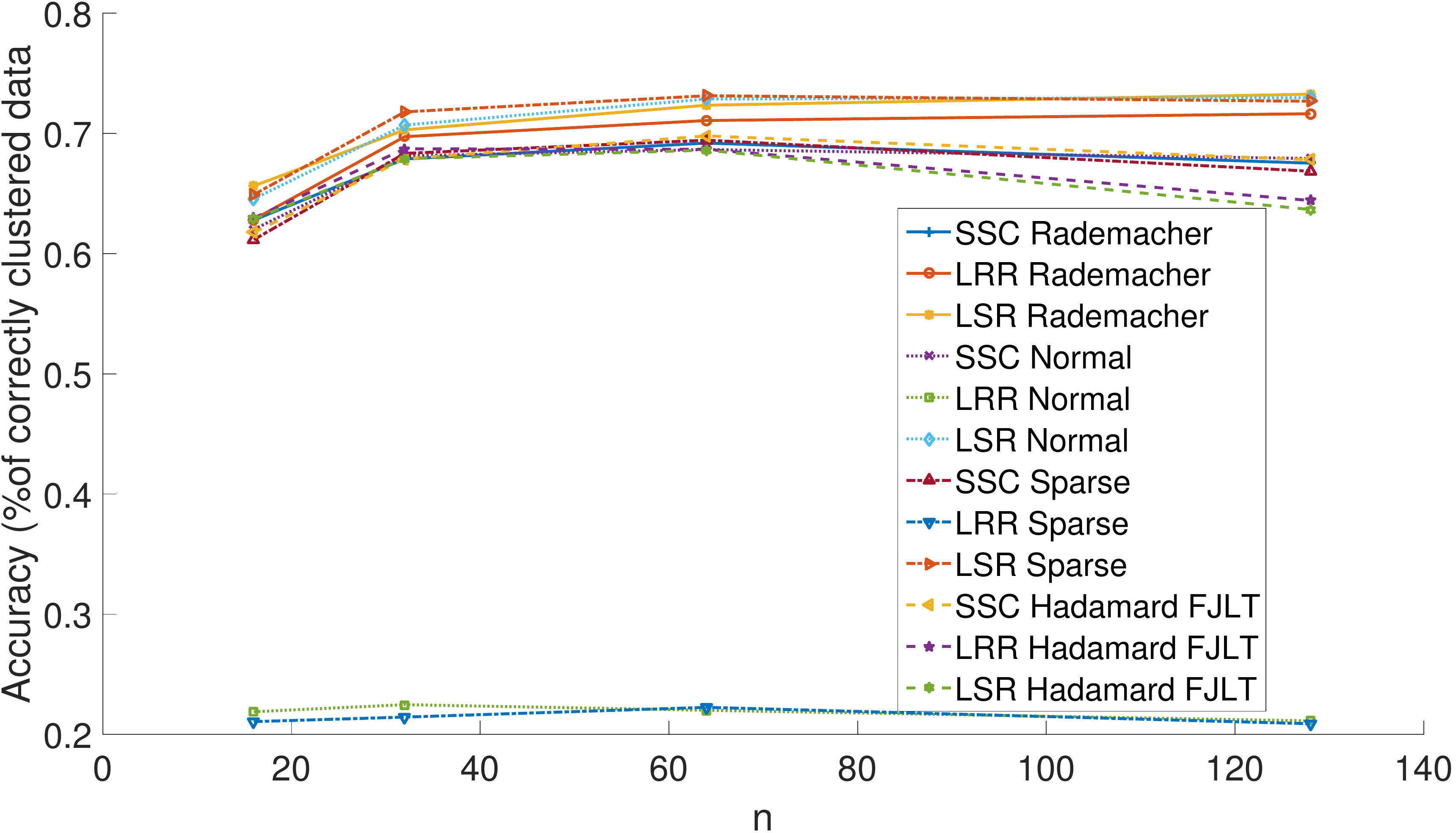}
			\label{fig:JLT_Coil}}
	\end{subfloat}
	
	\caption{Simulated tests on real datasets Extended Yale Face Database and COIL-100, evaluating the clustering performance with different JLT matrix $\mathbf{R}$.}\label{fig:JLT_test}
\end{figure}
\subsection{High volume of data}
\label{ssec:simulations_highvol}
In this section the performance of Sketch-SC (Alg.~\ref{alg}) is assessed on all datasets.
Hopkins 155 is a popular benchmark dataset for subspace clustering and motion segmentation. It contains 155 video sequences, with $N$ points tracked in each frame of a video sequence. Clusters ($K=2$ or $K=3$) represent different objects moving in the video sequence. The results
for the Hopkins 155 dataset are listed in Tab.~\ref{tab:Hopkins} for $K=2$ and $K=3$ clusters, with $n = 0.15N$ for the proposed methods. Here $\alpha = 800$ was used for SSC and $\alpha = 100$ for Sketch-SSC, $\lambda = 1$ for LRR and $\lambda = 10$ for Sketch-LRR, $\lambda = 4.6\cdot10^{-3}$ for LSR and Sketch-LSR. The number of nearest neighbors for Alg.~\ref{alg:knn} is set to $k=5$. As the size of the dataset is small, large computational gains are not expected by using Alg.~\ref{alg}. Nevertheless, the Sketch-SC methods achieve comparable accuracy to their batch counterparts, while in most cases (except one) requiring less time.

\begin{figure}[tb]
	\centering
	\begin{subfloat}[Clustering
		accuracy]{\includegraphics[width=0.5\columnwidth]{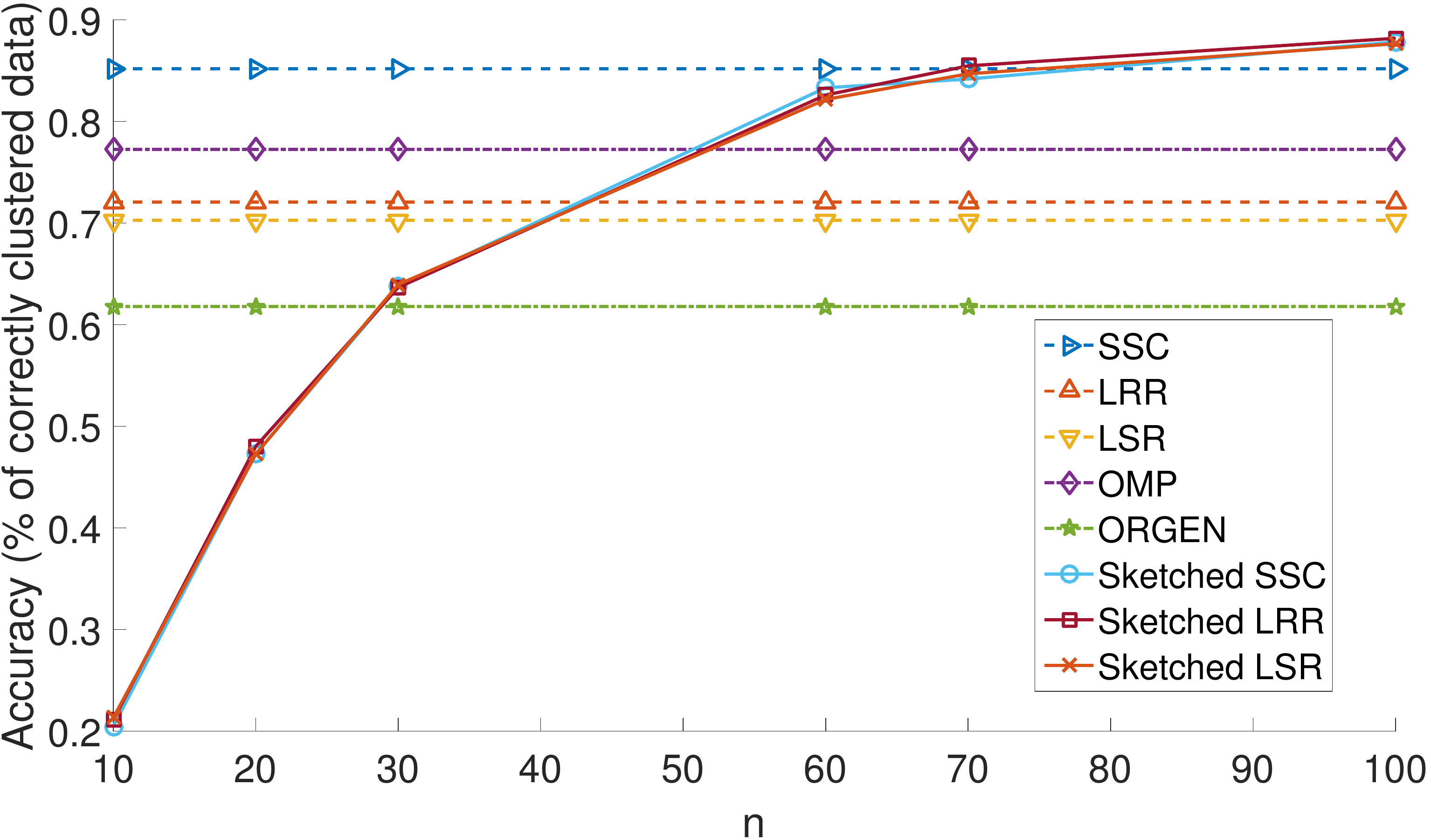} 
			\label{fig:Yale_accuracy}}
	\end{subfloat}
	
	\begin{subfloat}[Clustering time]{\includegraphics[width=0.5\columnwidth]{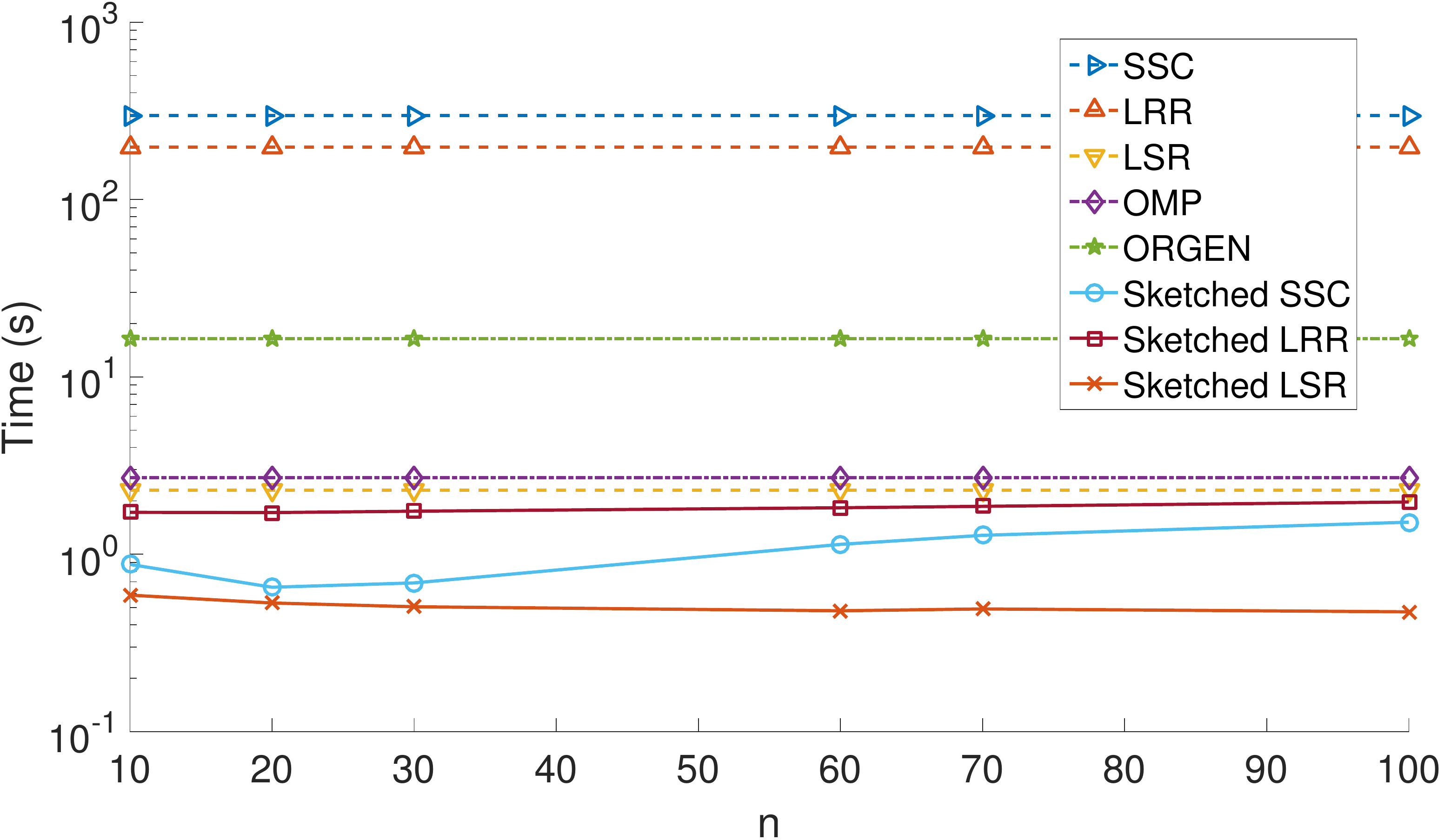}
			\label{fig:Yale_time}}
	\end{subfloat}
	
	\caption{Simulated tests on real dataset Extended Yale Face
		Database B, with $N=2,414$ data dimension $D=2,016$ and
		$K=38$ clusters for varying $n$.}\label{fig:yaleb}
\end{figure}

  \begin{table*}[tb]
  	\centering
  	 {
  	\begin{tabular}{ |c|c|c|c|c|c|c| }
  		\hline
  		\multicolumn{7}{|c|}{$K=2$} \\\hline
  		Algorithm & SSC & LRR & LSR & Sketch-SSC & Sketch-LRR & Sketch-LSR \\ \hline
  		Accuracy & $0.9839$ & $0.9723$ & $0.982$ & $\bf{0.946}$ & $\bf{0.9435}$ & $\bf{0.9319}$ \\  \hline
  		Time (s) & $0.6902$ & $0.9478$ & $0.093$ & $\bf{0.0795}$ & $\bf{0.0808}$ & $\bf{0.0787}$\\ 
  		\hline
  		\multicolumn{7}{|c|}{$K=3$} \\\hline
  		Algorithm & SSC & LRR & LSR & Sketch-SSC & Sketch-LRR & Sketch-LSR \\ \hline
  		Accuracy & $0.9747$ & $0.9253$ & $0.9654$ & $\bf{0.8942}$ & $\bf{0.9415}$ & $\bf{0.9242}$ \\  \hline
  		Time (s) & $1.566$ & $1.295$ & $0.1797$ & $\bf{0.1755}$ & $\bf{0.1459}$ & $\bf{0.1829}$\\ 
  		\hline
  	\end{tabular}
  	\bigskip
  	\caption{Results for $K=2$ and $K=3$ motions for the Hopkins155 dataset}
  	\label{tab:Hopkins}
  }
  \end{table*}
  
    \begin{table*}[tb]
    	\centering
    	{
    		\begin{tabular}{ |c | c |c|c|c|c|c| }
    			\hline
    			\multicolumn{2}{|c|}{Dataset}  & OMP & ORGEN  & Sketch-SSC & Sketch-LRR & Sketch-LSR \\ \hline
    			\multirow{ 2}{*}{MNIST}& Accuracy & $0.47049$ & $0.93788$  & $\bf{0.85825}$ & $\bf{0.90644}$ & $\bf{0.90784}$ \\  
    			& Time (s) & $502.91$ & $801.3954$  & $\bf{155.1017}$ & $\bf{156.7709}$ & $\bf{99.4724}$\\\hline 
    			\multirow{ 2}{*}{CoverType}& Accuracy & $0.4870$ & $0.4873$  & $\bf{0.42387}$  & $\bf{0.3277}$ & $\bf{0.4860}$ \\  
    			& Time (s) & $1.8947*10^4$ & $2.9893*10^4$  & $\bf{6064.8403}$ & $\bf{4468.5274}$ & $\bf{392.916}$\\\hline
    			\multirow{ 2}{*}{PokerHand}& Accuracy & $0.5009$ & \multirow{2}{*}{-}  & $\bf{0.5008}$ & $\bf{0.1833}$ & $\bf{0.44225}$ \\  
    			& Time (s) & $4.6654*10^4$ &   & $\bf{7.8*10^{3}}$ & $\bf{3.6*10^{4}}$ & $\bf{2.71*10^{3}}$\\\hline
    		\end{tabular}
    		\bigskip
    		\caption{Results for the Preprocessed MNIST dataset ($N=70,000$), the CoverType dataset ($N=581,012$) and the PokerHand dataset ($N=1,000,000$) }
    		\label{tab:Large_data}
    	}
    \end{table*}

The Extended Yale Face database contains $N=2,414$ face images of
$K=38$ people, each of dimension $D=2,016$. Fig.~\ref{fig:yaleb} shows the results for
this dataset for varying $n$, where $\alpha = 30$ for SSC and $\alpha = 50$ for Sketch-SSC, $\lambda = 0.15$ for LRR and  Sketch-LRR, $\lambda = 10^6$ for LSR and Sketch-LSR, the number of non-zeros per column of $\mathbf{Z}$ for OMP is set to $5$, while $\lambda=0.7$ and $\alpha = 200$ for ORGEN. The number of nearest neighbors for Alg.~\ref{alg:knn} is set to $k=5$. The proposed algorithms
 exhibit comparable accuracy to their batch counterparts, in particular SSC, and also achieve higher accuracy than the state-of-the-art large-scale algorithms OMP and ORGEN, as $n$ increases. Interestingly, with $n\approx0.03\cdot N$ the proposed methods achieve the accuracy of batch SSC. In addition, the proposed approach requires markedly less time than the batch methods, and less time than OMP and ORGEN as well.

The Columbia object-image dataset (COIL-100) contains $N=7,200$ images of size $32\times 32$ corresponding to $K=100$ objects. Each cluster corresponds to one object, and contains images of it from $72$ different angles. Fig.~\ref{fig:COIL} shows the comparisons on this dataset for varying $n$, where $\alpha = 25$ for SSC and $\alpha = 500$ for Sketch-SSC, $\lambda = 0.9$ for LRR and $\lambda = 10^{-4}$ for Sketch-LRR, $\lambda = 10^2$ for LSR and Sketch-LSR, the number of non-zeros per column of $\mathbf{Z}$ for OMP is set to $2$, while $\lambda=0.95$ and $\alpha = 3$ for ORGEN. The number of nearest neighbors for Alg.~\ref{alg:knn} is set to $k=5$. The proposed approaches exhibit performance comparable to the state-of-the-art as $n$ increases, while requiring significantly less time. Note that, OMP requires almost the same time as the proposed approaches, however its clustering performance is significantly lower.

{ 
	Fig.~\ref{fig:sval_coil_yale} plots the singular values of the Extended Yale Face Database and the COIL-$100$ dataset. For both, the largest singular values are approximately the first $70$ ones. Note that for the Extended Yale face database our proposed approaches attain their best performance for approximately $n=70$ yielding a compression ratio of $\frac{2414}{70}\approx 34.5$, while for the COIL-100 database  our proposed approaches reach their peak performance again for $n=70$, but this time the compression ratio is $\frac{7200}{70}\approx102.85$. This suggests that, indeed, datasets that exhibit low rank can be compressed with a lower $n$.
	}

\begin{figure}[tb]
	\centering
	\begin{subfloat}[Clustering
		accuracy]{\includegraphics[width=0.5\columnwidth]{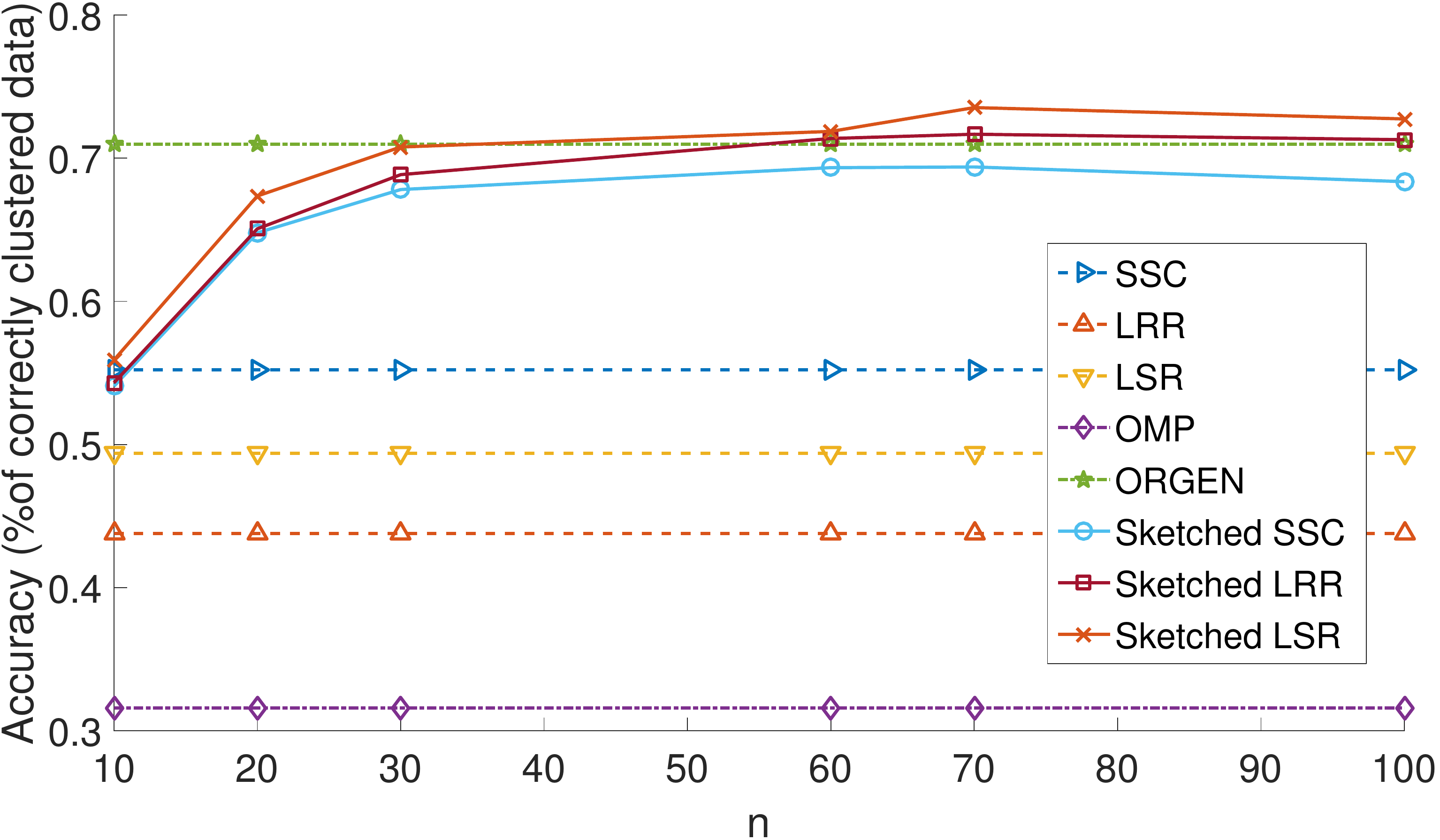} 
			\label{fig:COIL_accuracy}}
	\end{subfloat}
	
	\begin{subfloat}[Clustering time]{\includegraphics[width=0.5\columnwidth]{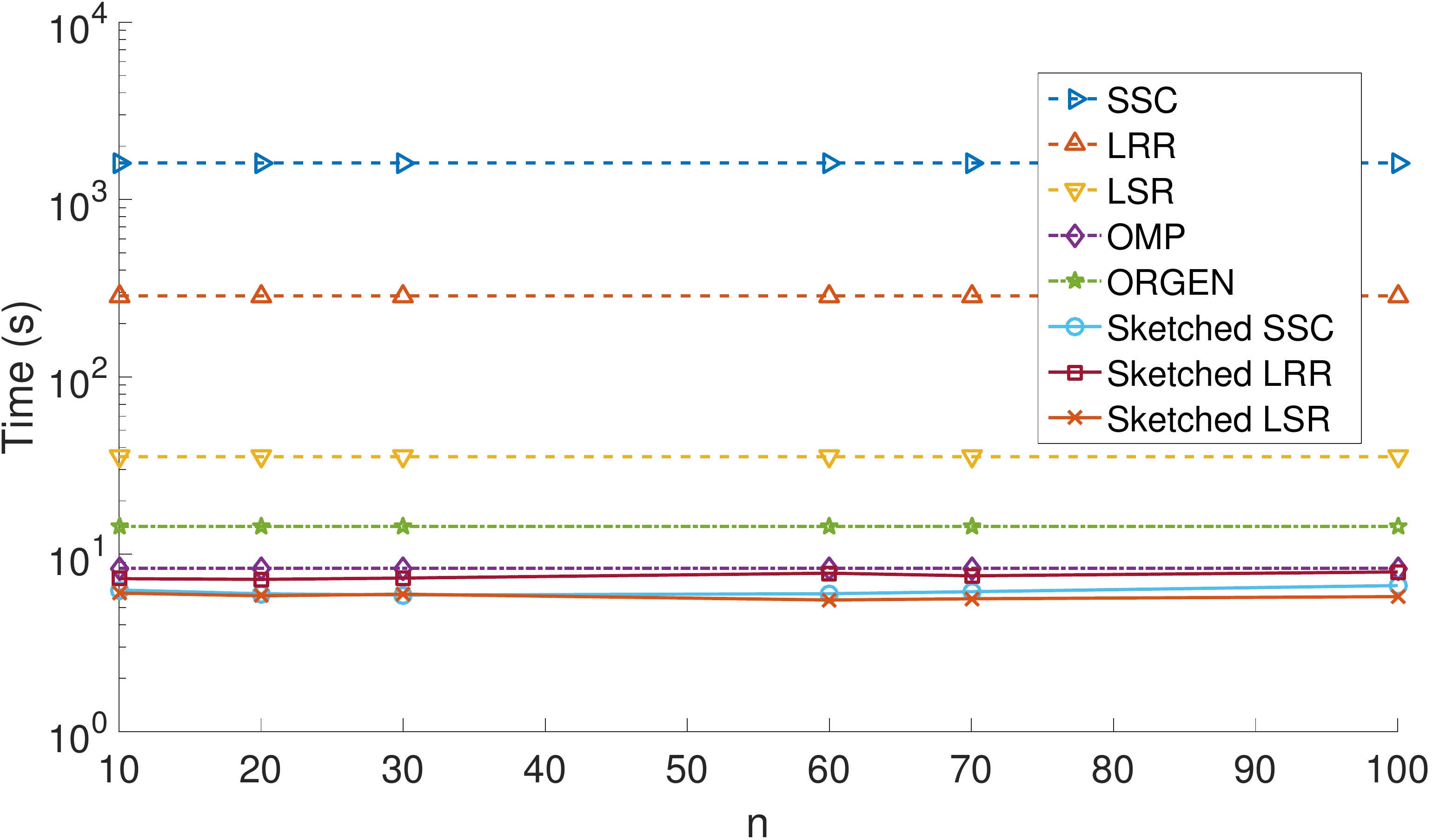}
			\label{fig:COIL_time}}
	\end{subfloat}
	
	\caption{Simulated tests on real dataset COIL-100, with $N=7,200$ data dimension $D=1,025$ and
		$K=100$ clusters for varying $n$.}\label{fig:COIL}
\end{figure}

\begin{figure}[tb]
	\centering
	\includegraphics[width=0.5\columnwidth]{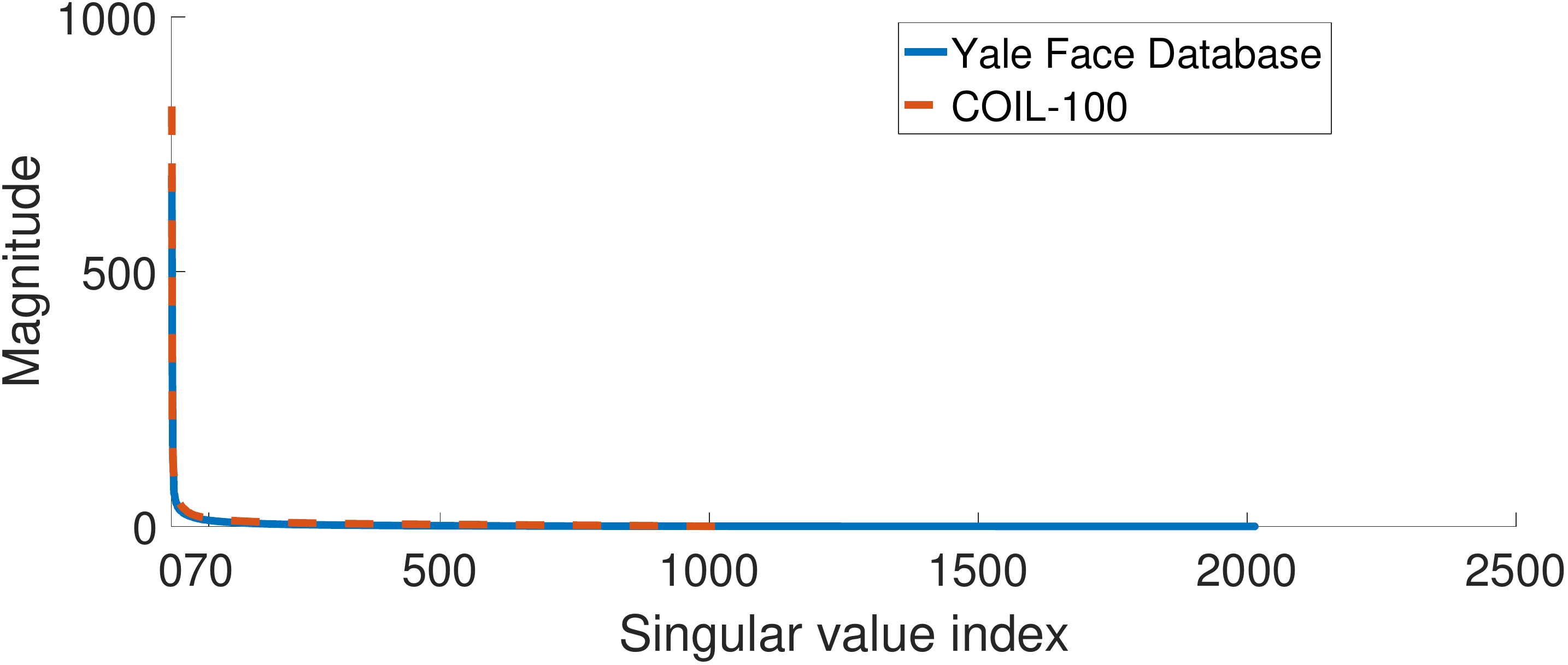}
	\caption{Singular value plots for the Extended Yale Face database and the COIL-100 dataset.}\label{fig:sval_coil_yale}
\end{figure}

{
Due to their large size, tests on the following three datasets compare Alg.~\ref{alg} only to OMP and ORGEN. The results for the following three datasets are listed in Tab.~\ref{tab:Large_data}. 
The MNIST dataset contains $70,000$ images of handwritten digits, each of dimension $28\times 28$, with $K=10$ clusters, one per digit. Here the dataset is preprocessed with a scattering convolutional network~\cite{ScatNet} and PCA to bring each image dimension down to $D=500$, as per \cite{VidalOMP,ORGEN}. Here $n=200$, $\alpha = 12,000$ for Sketch-SSC, $\lambda = 1$ for Sketch-LRR, $\lambda = 10^{-1}$ for Sketch-LSR, the number of non-zeros per column of $\mathbf{Z}$ for OMP is set to $10$, while $\lambda=0.95$ and $\alpha = 120$ for ORGEN. The number of nearest neighbors for Alg.~\ref{alg:knn} is set to $k=3$, and the set of nearest neighbors for each datum is found using the ANN implementation of the VLfeat package. In this scenario ORGEN showcases the best clustering performance, however Sketch-LRR and Sketch-LSR exhibit comparable accuracy, while requiring markedly less time. 

The CoverType dataset consists of $N=581,012$ data belonging to $K=7$ clusters. Each cluster corresponds to a different forest cover type. Data are vectors of dimension $D=54$ that contain cartographic variables, such as soil type, elevation, hillshade etc. Here $n=150$, $\alpha = 1$ for Sketch-SSC, $\lambda = 10^{-8}$ for Sketch-LRR, $\lambda = 10^4$ for Sketch-LSR, the number of non-zeros per column of $\mathbf{Z}$ for OMP is set to $15$, while $\lambda=0.95$ and $\alpha = 500$ for ORGEN. The number of nearest neighbors for Alg.~\ref{alg:knn} is set to $k=10$, and the set of nearest neighbors for each datum is found using the ANN implementation of the VLfeat package.

The PokerHand database contains $N=10^6$ data, belonging to $K=10$
classes. Each datum is a $5$-card hand drawn from a deck of $52$
cards, with each card being described by its suit (spades, hearts,
diamonds, and clubs) and rank (Ace, 2, 3, \ldots, Queen, King). Each class represents a valid Poker
hand. Here $n=30$, $\alpha = 10$ for Sketch-SSC, $\lambda = 1$ for Sketch-LRR, $\lambda = 10^2$ for Sketch-LSR, the number of non-zeros per column of $\mathbf{Z}$ for OMP is set to $10$. The number of nearest neighbors for Alg.~\ref{alg:knn} is set to $k=20$, and the set of nearest neighbors for each datum is found using the ANN implementation of the VLfeat package.
Results are not reported for ORGEN as the algorithm did not converge within $24$ hours. For both the CoverType and PokerHand datasets, most algorithms exhibit comparable accuracy, while Alg.~\ref{alg} requires again less time than OMP or ORGEN.
}
 
 \subsection{High-dimensional data}
 In this section, the performance of Sketch-SC approaches combined with randomized dimensionality reduction (Alg.~\ref{alg:highD}) is assessed, for the Extended Yale Face database.
 
 Fig.~\ref{fig:yaleb_dr} depicts the simulation results on the Extended Yale Face database, when performing dimensionality reduction, for varying $d$. Here Alg.~\ref{alg:highD}, with fixed $n=70$ is compared to its batch counterparts, OMP and ORGEN. LRR and Sketch-LRR are not included in this simulation as the algorithm failed for small values of $d$. All parameters are the same as the corresponding experiment in Sec.~\ref{ssec:simulations_highvol}. In this experiment, Sketch-LSR and Sketch-SSC outperform their competing alternatives in terms of clustering accuracy, while maintaining a low computational overhead. OMP also exhibits low computational time, at the expense of clustering accuracy. 
\begin{figure}[tb]
	\centering
	\begin{subfloat}[Clustering
		accuracy]{\includegraphics[width=0.5\columnwidth]{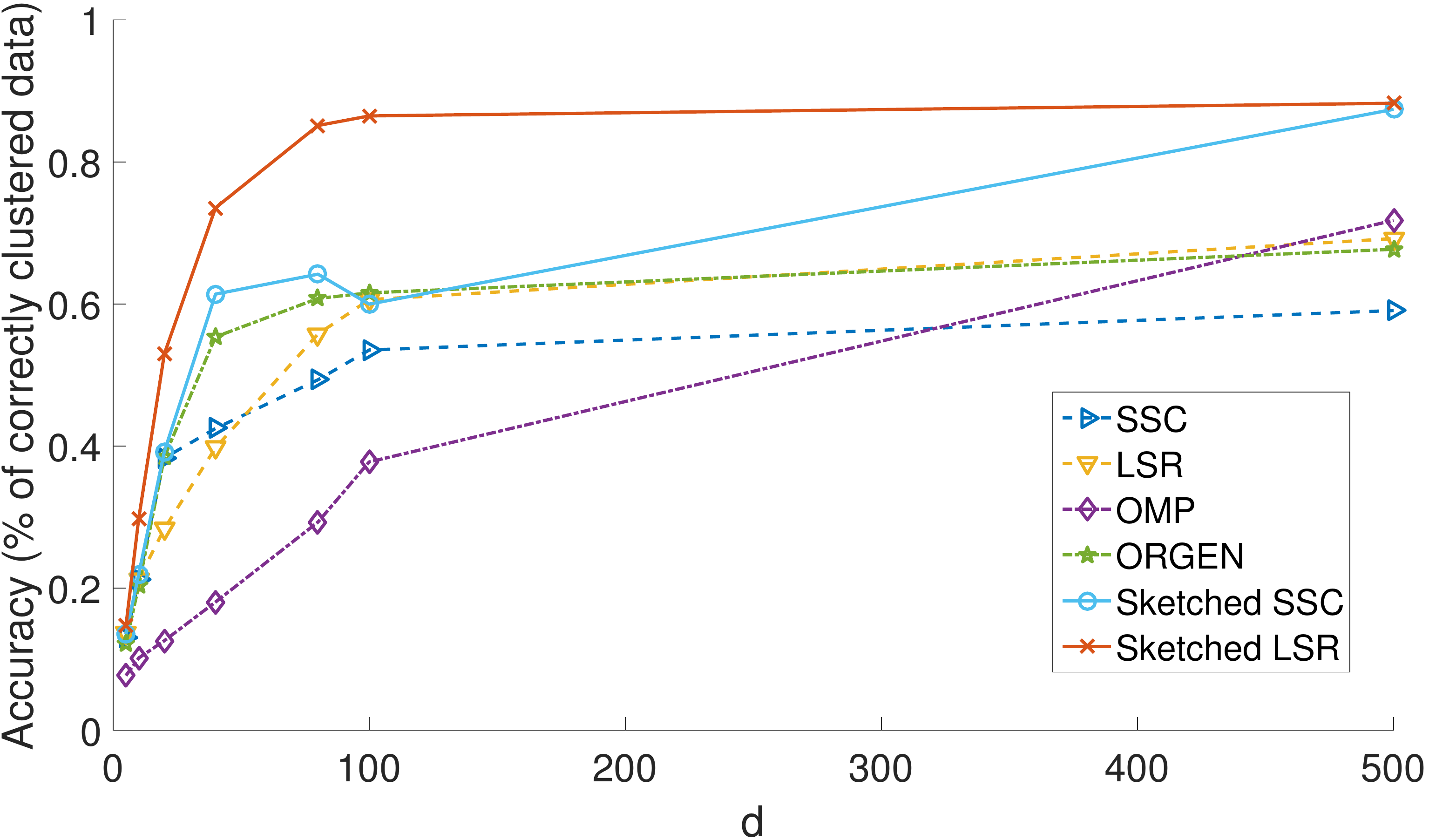} 
			\label{fig:Yale_accuracy_dr}}
	\end{subfloat}
	
	\begin{subfloat}[Clustering time]{\includegraphics[width=0.5\columnwidth]{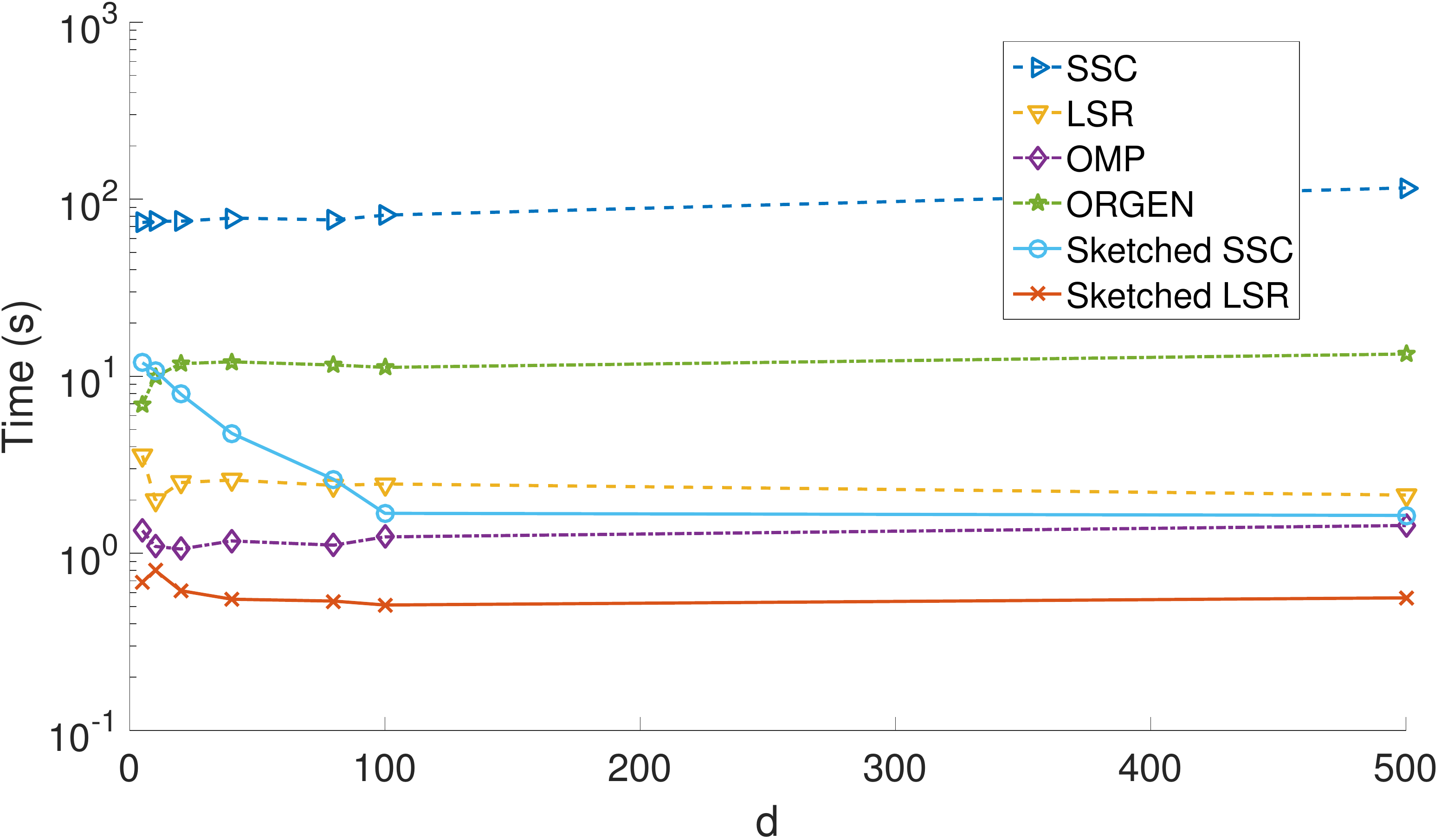}
			\label{fig:Yale_time_dr}}
	\end{subfloat}
	
	\caption{Simulated tests on real dataset Extended Yale Face
		Database B, with $N=2,414$ data dimension $D=2,016$ and
		$K=38$ clusters for varying $d$ and fixed $n=70$.}\label{fig:yaleb_dr}
\end{figure}

\section{Conclusions and future work}\label{sec:conclusion}
\squeezeup
The present paper introduced a novel data-reduction scheme for subspace clustering, namely
Sketch-SC, that enables grouping of data drawn from a union of
subspaces based on a random sketching approach for
fast, yet-accurate subspace clustering. Performance of the proposed scheme was evaluated both analytically and through simulated tests on multiple real datasets. Future research directions
will focus on the development of online and distributed Sketch-SC, able to handle not
only big, but also fast-streaming data. In addition, the sketched SC approach could be generalized to subspace clustering for tensor data.

\appendices
\section{Technical proofs}\label{app:proofs} 
\setcounter{proposition}{0}
\setcounter{corollary}{0}
\setcounter{theorem}{0}

\subsection{Supporting Lemmata}
\label{app:ssec:lemmata}
The following lemmata will be used to assist in the proofs of the propositions and theorems.
\begin{lemma}{\cite[Corollary 11]{sarlos}}\emph{
	\label{lemma:sval}
	Consider an  $N\times k$ orthonormal matrix $\mathbf{V}$ with $N\geq k$, and
   a JLT($\varepsilon,\delta,k$) matrix $\mathbf{R}$ of size $N\times n$. If $n=\mathcal{O}(k\frac{\log(k/\varepsilon)}{\varepsilon^2} f(\delta))$, then the following holds w.p. at least $1-\delta$
	\begin{equation}
	1 - \varepsilon \leq \sigma_i^2(\mathbf{V}^{\top}\mathbf{R})\leq 1 + \varepsilon \quad\text{ for } i=1,\ldots,k
	\end{equation}
	where $\sigma_i(\mathbf{V}^{\top}\mathbf{R})$ denotes the $i$-th singular value of $\mathbf{V}^{\top}\mathbf{R}$.
}
\end{lemma}

\begin{lemma}{\cite[Lemma 8]{boutsidis}}
	\emph{
	\label{lemma:pseudo_transpose} Let $\varepsilon>0$, and consider the $n\times k$ orthonormal matrix $\mathbf{V}$ with $n>k$, as well as the $n\times r$ matrix $\mathbf{R}$, with $r>k$ satisfying $1 - \varepsilon \leq \sigma_i^2(\mathbf{V}^{\top}\mathbf{R})\leq 1 + \varepsilon$ for $i=1,\ldots,k$. It then holds deterministically that
	\begin{equation}
	\|(\mathbf{V}^{\top}\mathbf{R})^{\dagger} - (\mathbf{V}^{\top}\mathbf{R})^{\top}\|_2 \leq \frac{\varepsilon}{\sqrt{1-\varepsilon}}.
	\end{equation}}
\end{lemma}

\subsection{Main proofs}
\setcounter{lemma}{1}
\begin{proposition}
	\emph{
	Let $\mathbf{X}$ be a $D\times N$ matrix such that $\text{rank}(\mathbf{X}) = \rho$, and define the $D\times n$ matrix $\mathbf{B}:=\mathbf{X}\mathbf{R}$, where $\mathbf{R}$ is a JLT($\varepsilon,\delta,D$) of size $N\times n$. If $n=\mathcal{O}(\rho\frac{\log(\rho/\varepsilon)}{\varepsilon^2}f(\delta))$ then w.p. at least $1-\delta$, it holds that 
	\begin{equation*}
	\text{range}(\mathbf{X}) = \text{range}(\mathbf{B}).
	\end{equation*}}
\end{proposition}
\begin{proof}
Let $\mathbf{X} = \mathbf{U}_\rho\mathbf{\Sigma}_\rho\mathbf{V}_\rho^{\top}$ be the SVD of $\mathbf{X}$. Since $\mathbf{V}_\rho$ is invertible and $\mathbf{\Sigma}_\rho$ is diagonal, it holds that 
\begin{equation}
\label{eq:range}
\text{range}(\mathbf{X}) = \text{range}(\mathbf{U}_\rho)
\end{equation}
i.e., the columns of $\mathbf{X}$ can be written as linear combinations of the columns of $\mathbf{U}_\rho$ and vice versa. 
Now consider $\mathbf{B} = \mathbf{X}\mathbf{R} = \mathbf{U}_\rho\mathbf{\Sigma}_\rho\mathbf{V}_\rho^{\top}\mathbf{R} = \mathbf{U}_\rho\mathbf{\Sigma}_\rho\tilde{\mathbf{V}}_\rho^{\top}$, where $\tilde{\mathbf{V}}_\rho := \mathbf{R}^{\top}\mathbf{V}_\rho$, which implies $\text{range}(\mathbf{B}) \subseteq \text{range}(\mathbf{U}_\rho)$. 
By Lemma \ref{lemma:sval} $\tilde{\mathbf{V}}_\rho^{\top} := \mathbf{V}_\rho^{\top}\mathbf{R}$ is full row rank w.p. at least $1-\delta$ and thus 
\begin{equation}
\mathbf{B}(\tilde{\mathbf{V}}^{\top})^{\dagger} = \mathbf{U}_\rho\mathbf{\Sigma}_\rho
\end{equation} 
which implies that $\text{range}(\mathbf{U}_\rho) = \text{range}(\mathbf{B}) = \text{range}(\mathbf{X})$, where the last equality is due to \eqref{eq:range}.
\end{proof}
\begin{proposition}
	\emph{
		Let $\mathbf{X}$ be a $D\times N$ matrix such that $\text{rank}(\mathbf{X}) = \rho$, and define the $D\times n$ matrix $\mathbf{B}:=\mathbf{X}\mathbf{R}$, where $\mathbf{R}$ is a JLT($\varepsilon,\delta,D$) of size $N\times n$. If $n=\mathcal{O}(r\frac{\log(r/\varepsilon)}{\varepsilon^2}f(\delta))$, then w.p. at least $1-2\delta$ it holds that
		\begin{equation*}
		\|\mathbf{B}(\mathbf{V}_r^{\top}\mathbf{R})^{\dagger} - \mathbf{U}_r\mathbf{\Sigma}_r\|_F \leq (\varepsilon\frac{\sqrt{1+\varepsilon}}{\sqrt{1-\varepsilon}}+1+\varepsilon)\|\bar{\mathbf{X}}_{r}\|_F.
		\end{equation*}
	}
\end{proposition}
\begin{proof}
	From the first part of the proof of Prop.~\ref{prop:full_range} we have that range($\mathbf{B}$) $\subseteq$ range($\mathbf{U}_\rho$). Now consider
	\begin{equation}
	\label{eq:prop2_D1}
	\mathbf{B} = \mathbf{X}\mathbf{R} = \mathbf{U}_r\mathbf{\Sigma}_r\mathbf{V}_r^{\top}\mathbf{R} + \bar{\mathbf{U}}_{r}\bar{\mathbf{\Sigma}}_{r}\bar{\mathbf{V}}_{r}^{\top}\mathbf{R}
	\end{equation}
	By Lemma~\ref{lemma:sval} $\mathbf{V}_r^{\top}\mathbf{R}$ is full row rank w.p. at least $1-\delta$; thus, right multiplying \eqref{eq:prop2_D1} with $(\mathbf{V}_r^{\top}\mathbf{R})^{\dagger}$  yields
	$\mathbf{B}(\mathbf{V}_r^{\top}\mathbf{R})^{\dagger} = \mathbf{U}_r\mathbf{\Sigma}_r + \bar{\mathbf{U}}_{r}\bar{\mathbf{\Sigma}}_{r}\bar{\mathbf{V}}_{r}^{\top}\mathbf{R}(\mathbf{V}_r^{\top}\mathbf{R})^{\dagger}$, or
	\begin{alignat*}{1}
	& \mathbf{B}(\mathbf{V}_r^{\top}\mathbf{R})^{\dagger} - \mathbf{U}_r\mathbf{\Sigma}_r = \bar{\mathbf{U}}_{r}\bar{\mathbf{\Sigma}}_{r}\bar{\mathbf{V}}_{r}^{\top}\mathbf{R}(\mathbf{V}_r^{\top}\mathbf{R})^{\dagger} 
	\end{alignat*}
	which upon substituting $\bar{\mathbf{X}}_{r}$ boils down to
		\begin{alignat}{1}
		\label{eq:prop2_D2}
		& \mathbf{B}(\mathbf{V}_r^{\top}\mathbf{R})^{\dagger} - \mathbf{U}_r\mathbf{\Sigma}_r  \\ & =\bar{\mathbf{X}}_{r}\mathbf{R}(\mathbf{V}_r^{\top}\mathbf{R})^{\dagger} - \bar{\mathbf{X}}_{r}\mathbf{R}(\mathbf{V}_r^{\top}\mathbf{R})^{\top} +  \bar{\mathbf{X}}_{r}\mathbf{R}(\mathbf{V}_r^{\top}\mathbf{R})^{\top.} \notag
		\end{alignat}
	Using the triangle inequality, and the spectral submultiplicativity of the Frobenius norm, yields
	\begin{alignat}{1}
	\label{eq:prop2_D3}
	& \|\mathbf{B}(\mathbf{V}_r^{\top}\mathbf{R})^{\dagger} - \mathbf{U}_r\mathbf{\Sigma}_r\|_F  \\ & = \|\bar{\mathbf{X}}_{r}\mathbf{R}\left((\mathbf{V}_r^{\top}\mathbf{R})^{\dagger} - (\mathbf{V}_r^{\top}\mathbf{R})^{\top}\right)\|_F +  \|\bar{\mathbf{X}}_{r}\mathbf{R}(\mathbf{V}_r^{\top}\mathbf{R})^{\top}\|_F \notag \\ 
	& \leq \|\bar{\mathbf{X}}_{r}\mathbf{R}\|_F~\|(\mathbf{V}_r^{\top}\mathbf{R})^{\dagger} - (\mathbf{V}_r^{\top}\mathbf{R})^{\top}\|_2  \notag\\ & +  \|\bar{\mathbf{X}}_{r}\mathbf{R}\|_F~\|(\mathbf{V}_r^{\top}\mathbf{R})^{\top}\|_2. \notag
	\end{alignat}
	We have from Def.~\ref{def:JLT} $\|\bar{\mathbf{X}}_{r}\mathbf{R}\|_F \leq \sqrt{1+\varepsilon}\|\bar{\mathbf{X}}_{r}\|_F$ w.p. at least $1-\delta$, while Lemma \ref{lemma:sval} ensures  $\|(\mathbf{V}_r^{\top}\mathbf{R})^{\top}\|_2\leq\sqrt{1+\varepsilon}$ w.p. at least $1-\delta$. Since Lemma~\ref{lemma:pseudo_transpose} also implies that $\|(\mathbf{V}_r^{\top}\mathbf{R})^{\dagger} - (\mathbf{V}_r^{\top}\mathbf{R})^{\top}\|_2\leq \frac{\varepsilon}{\sqrt{1-\varepsilon}}$ we arrive at [cf.~\ref{eq:prop2_D3}]
	\begin{equation}
	\|\mathbf{B}(\mathbf{V}_r^{\top}\mathbf{R})^{\dagger} - \mathbf{U}_r\mathbf{\Sigma}_r\|_F \leq (\varepsilon\frac{\sqrt{1+\varepsilon}}{\sqrt{1-\varepsilon}}+1+\varepsilon)\|\bar{\mathbf{X}}_{r}\|_F.
	\end{equation}
\end{proof}
\begin{theorem}\emph{ Consider noise-free and normalized data vectors obeying \eqref{eq:allpoints} with $\bm{v}_i \equiv 0$, to form columns of a $D\times N$ data matrix $\mathbf{X}$, with unit $\ell_2$ norm per column, and $\text{rank}(\mathbf{X}) = \rho$. Let also $\mathbf{R}$ denote JLT($\varepsilon,\delta,D$) of size $N\times n$. Let $\bm{g}^*(\bm{x}) := \mathbf{X}\bm{a}^* = \bm{x}$ denote the ground-truth representation of $\bm{x}$, and $\hat{\bm{g}}(\bm{x}) := \mathbf{X}\mathbf{R}\hat{\bm{a}}$ the representation given by Sketch-LSR. If $n=\mathcal{O}(r\frac{\log(r/\varepsilon)}{\varepsilon^2}f(\delta))$, then the following bound holds w.p. at least $1-2\delta$
	\begin{equation*}
	\|\bm{g}^{*}(\bm{x}) - \hat{\bm{g}}(\bm{x})\|_2 \leq \lambda~(1 + \sqrt{\frac{{1+\varepsilon}}{1-\varepsilon}}~\sqrt{\rho - r}~\sigma_{r+1}^2) + \frac{1}{\sqrt{1-\varepsilon}}
	\end{equation*}
	with $\lambda$ as in \eqref{eq:general_sc_eq}, and $\sigma_{r+1}$ denotes the $(r+1)$st singular value of $\mathbf{X}$.}
\end{theorem}
\begin{proof}
	The proof will follow the steps in~\cite{yang2015randomized}. Consider the Sketch-LSR objective for $\bm{x}$, namely
	\begin{equation}\label{eq:prop1}
	\frac{\lambda}{2}\|\bm{x} - \mathbf{X}\mathbf{R}\bm{a}\|_2^2 + \|\bm{a}\|_2^2
	\end{equation}
	and the SVD $\mathbf{X} = \mathbf{U}\mathbf{\Sigma}\mathbf{V}^{\top}$. As $\mathbf{U}$ is unitary, minimizing \eqref{eq:prop1} is equivalent to minimizing 
	\begin{equation}\label{eq:prop2}
	\frac{\lambda}{2}\|\mathbf{U}^{\top}\bm{x} - \mathbf{\Sigma}\tilde{\mathbf{V}}^{\top}\bm{a}\|_2^2 + \|\bm{a}\|_2^2
	\end{equation}
	where $\tilde{\mathbf{V}}^{\top} := \mathbf{V}^{\top}\mathbf{R}$. Now, decompose the dataset as 
	\begin{equation}
	\label{eq:thm_decomp}
	\begin{aligned}
	\mathbf{X} =  \mathbf{X}_r + \bar{\mathbf{X}}_{r}
	\end{aligned}
	\end{equation}
	where  $\mathbf{X}_r := \mathbf{U}_r\mathbf{\Sigma}_r\mathbf{V}_r^{\top}$ and  $\bar{\mathbf{X}}_{r} := \bar{\mathbf{U}}_{r}\bar{\mathbf{\Sigma}}_{r}\bar{\mathbf{V}}_{r}^{\top}$. Using \eqref{eq:thm_decomp}, we can rewrite \eqref{eq:prop2} as 
	\begin{equation}\label{eq:prop3}
	\frac{\lambda}{2}\underbrace{\|\bm{\chi}_r - \mathbf{\Sigma}_r\tilde{\mathbf{V}}_r^{\top}\bm{a}\|_2^2}_{:=T_1^2} +
	\frac{\lambda}{2}\underbrace{\|\bar{\bm{\chi}}_{r} - \bar{\mathbf{\Sigma}}_{r}\bar{\tilde{\mathbf{V}}}_{r}^{\top}\bm{a}\|_2^2}_{:=T_2^2} + \underbrace{\|\bm{a}\|_2^2}_{:=T_3^2}
	\end{equation}
	where $\bm{\chi}_r := \mathbf{U}_r^{\top}\bm{x}$, and $\bar{\bm{\chi}}_{r} := \bar{\mathbf{U}}_{r}^{\top}\bm{x}$. Selecting $\bm{a}$ as 
	\begin{equation*}
	\bm{a} = \tilde{\mathbf{V}}_r(\tilde{\mathbf{V}}_r^{\top}\tilde{\mathbf{V}}_r)^{-1}\mathbf{\Sigma}_r^{-1}\bm{\chi}_r
	\end{equation*}
	$T_1^2$ vanishes, and $T_2$ reduces to
	\begin{equation}
	T_2 = \|\bar{\bm{\chi}}_{r} - \bar{\mathbf{\Sigma}}_{r}\bar{\tilde{\mathbf{V}}}_{r}^{\top}\tilde{\mathbf{V}}_r(\tilde{\mathbf{V}}_r^{\top}\tilde{\mathbf{V}}_r)^{-1}\mathbf{\Sigma}_r^{-1}\bm{\chi}_r\|_2.
	\end{equation}
	The triangle inequality and the submultiplicativity of the $\ell_2$ norm, allows us to bound $T_2$ as
	\begin{equation}
	\label{eq:thm_T2bnd}
	\begin{aligned}
	& T_2 \leq \|\bar{\bm{\chi}}_{r}\|_2  \\ & + \|\bar{\mathbf{\Sigma}}_{r}\bar{\tilde{\mathbf{V}}}_{r}^{\top}\|_2~\|\tilde{\mathbf{V}}_r(\tilde{\mathbf{V}}_r^{\top}\tilde{\mathbf{V}}_r)^{-1}\|_2\|\mathbf{\Sigma}_r^{-1}\bm{\chi}_r\|_2.
	\end{aligned}
	\end{equation}
	Now note that $\|\bar{\mathbf{\Sigma}}_{r}\bar{\tilde{\mathbf{V}}}_{ r}^{\top}\|_2 \leq \|\bar{\mathbf{\Sigma}}_{r}\bar{\tilde{\mathbf{V}}}_{ r}^{\top}\|_F = \|\bar{\mathbf{U}}_{r}\bar{\mathbf{\Sigma}}_{r}\bar{\tilde{\mathbf{V}}}_{ r}^{\top}\|_F = \|\bar{\mathbf{X}}_{r}\mathbf{R}\|_F$ and recall from Def.~\ref{def:JLT} that $\|\bar{\mathbf{X}}_{r}\mathbf{R}\|_F \leq \sqrt{1+\varepsilon}\|\bar{\mathbf{X}}_{r}\|_F\leq \sqrt{1+\varepsilon}\sqrt{\rho-r}\|\bar{\mathbf{X}}_{r}\|_2 \leq \sqrt{1+\varepsilon}\sqrt{\rho-r}~\sigma_{r+1}^2$ w.p. at least $1-\delta$. By Lemma~\ref{lemma:sval} $\tilde{\mathbf{V}}_r^{\top} = \mathbf{V}_r^{\top}\mathbf{R}$ is full row rank w.p. at least $1-\delta$; thus, $\tilde{\mathbf{V}}_r(\tilde{\mathbf{V}}_r^{\top}\tilde{\mathbf{V}}_r)^{-1} = \tilde{\mathbf{V}}_r^{\dagger}$, and $\|\tilde{\mathbf{V}}_r^{\dagger}\|_2\leq \frac{1}{\sqrt{1-\varepsilon}}$. Furthermore, $\|\mathbf{\Sigma}_r^{-1}\bm{\chi}_r\|_2 = \|\mathbf{V}_r\mathbf{\Sigma}_r^{-1}\mathbf{U}_r^{\top}\bm{x}\|_2\leq 1$, and $\|\bar{\bm{\chi}}_{r}\|_2 = \|\bar{\mathbf{U}}_{r}^{\top}\bm{x}\|_2 = 1$. Similarly, $T_3$ in \eqref{eq:prop3} can be bounded w.p. at least $1-\delta$ due to Lemma \ref{lemma:sval} as
	\begin{alignat}{2 }
	\label{eq:thm_T3bnd}
	& T_3  && = \|\tilde{\mathbf{V}}_r(\tilde{\mathbf{V}}_r^{\top}\tilde{\mathbf{V}}_r)^{-1}\mathbf{\Sigma}_r^{-1}\bm{\chi}_r\|_2 = \|\tilde{\mathbf{V}}_r^{\dagger}\mathbf{\Sigma}_r^{-1}\bm{\chi}_r \|_2 \\
	& && \leq \|\tilde{\mathbf{V}}_r^{\dagger}\|_2~\|\mathbf{\Sigma}_r^{-1}\bm{\chi}_r\|_2\leq \frac{1}{\sqrt{1 - \varepsilon}} \notag
	\end{alignat}
	 Finally, since the chosen $\bm{a}$ in \eqref{eq:prop3} satisfies \eqref{eq:thm_T2bnd} and \eqref{eq:thm_T3bnd}, so will do any minimizer $\hat{\bm{a}}$ of \eqref{eq:prop1}.
\end{proof}
{
\begin{corollary}
	\emph{
		Consider the setting of Thm.~\ref{thm:LSR}, and let $\hat{\bm{g}}(\bm{x}) := \mathbf{X}\mathbf{R}\hat{\bm{a}}$ be the representation of a datum given by Sketch-SSC.
		The following bound holds w.p. at least $1-2\delta$
		\begin{equation*}
		\|\bm{g}^{*}(\bm{x}) - \hat{\bm{g}}(\bm{x})\|_2 \leq \lambda~(1 + \sqrt{\frac{{1+\varepsilon}}{1-\varepsilon}}~\sqrt{\rho - r}~\sigma_{r+1}^2) + \sqrt{\frac{n}{{1-\varepsilon}}}
		\end{equation*}
		with $\lambda$ as in \eqref{eq:general_sc_eq}, and $\sigma_{r+1}$ denotes the $(r+1)$st singular value of $\mathbf{X}$.}
\end{corollary}
\begin{proof}
	Consider the Sketch-SSC objective for $\bm{x}$, namely
	\begin{equation}\label{eq:corr11}
	\frac{\lambda}{2}\underset{{:=T_1^2}}{\underbrace{\|\bm{x} - \mathbf{X}\mathbf{R}\bm{a}\|_2^2}} + \underset{:=T_2}{\underbrace{\|\bm{a}\|_1}}.
	\end{equation}
	From Thm.~\ref{thm:LSR} we have $T_1 \leq \lambda~(1 + \sqrt{\frac{{1+\varepsilon}}{1-\varepsilon}}~\sqrt{\rho - r}~\sigma_{r+1}^2)$, and $\|\bm{a}\|_2\leq \frac{1}{\sqrt{1 - \varepsilon}}$. Since for any $n\times 1$ vector $\bm{z}$ it holds that $\|\bm{z}\|_1 \leq \sqrt{n}\|\bm{z}\|_2$, we have $T_2\leq\sqrt{n}\|\bm{a}\|_2\leq \sqrt{\frac{n}{1-\varepsilon}}$ yielding the claim of the corollary.
\end{proof}}
{
	\begin{corollary}
		\emph{
			Consider the setting of Thm.~\ref{thm:LSR}, and let ${\bm{g}}^{*}(\mathbf{X}) := \mathbf{X}{\mathbf{Z}}$ and $\hat{\bm{g}}(\mathbf{X}) := \mathbf{X}\mathbf{R}\hat{\mathbf{A}}$ be the representations of all the data given by LRR and Sketch-LRR respectively.
			The following bound holds w.p. at least $1-2\delta$
			{\small\begin{equation*}
				\|\bm{g}^{*}(\mathbf{X}) - \hat{\bm{g}}(\mathbf{X})\|_F \leq \lambda~(\sqrt{N} + \sqrt{\frac{{1+\varepsilon}}{1-\varepsilon}}~\sqrt{\rho - r}~\sigma_{r+1}^2) + \sqrt{\frac{n}{{1-\varepsilon}}}
				\end{equation*}}
			with $\lambda$ as in \eqref{eq:general_sc_eq}, and $\sigma_{r+1}$ denoting the $(r+1)$st singular value of $\mathbf{X}$.
		}
	\end{corollary}
	\begin{proof}
		Consider the Sketch-LRR objective for $\mathbf{X}$, namely
		\begin{equation}\label{eq:corr21}
		\frac{\lambda}{2}\underset{{:=T_1^2}}{\underbrace{\|\mathbf{X} - \mathbf{X}\mathbf{R}\mathbf{A}\|_F^2}} + \underset{:=T_2}{\underbrace{\|\mathbf{A}\|_*}}.
		\end{equation}
		As with Corr.~\ref{corr:ssc}, $T_1$ can be bounded using the results of Thm.~\ref{thm:LSR}, and $\|\mathbf{A}\|_F\leq\frac{1}{\sqrt{1-\varepsilon}}$. Since for any rank $n$ matrix $\mathbf{Z}$ it holds that $\|\mathbf{Z}\|_*\leq \sqrt{n}\|\mathbf{Z}\|_F$ we have $T_2\leq\sqrt{n}\|\mathbf{A}\|_F\leq\sqrt{\frac{n}{1-\varepsilon}}$, yielding the claim of the corollary.
	\end{proof}
	}
\begin{proposition}
\emph{	Consider $\bm{x}_i= \mathbf{X}\bm{z}_i$ and $\bm{x}_j= \mathbf{X}\bm{z}_j$, and their representation provided by SSC, LRR or LSR $\bm{z}_i$ and $\bm{z}_j$, respectively. Let $\rho = \text{rank}(\mathbf{X})$ and $\bm{a}_i$, $\bm{a}_j$ be the representation obtained by the corresponding Sketch algorithm of Section~\ref{sec:proposed_algorithms}; that is, $\bm{x}_i = \mathbf{X}\mathbf{R}\bm{a}_i$, where the $N\times n$ matrix $\mathbf{R}$ is a JLT($\varepsilon,\delta,D$). If $n=\mathcal{O}(\rho\frac{\log(\rho/\varepsilon)}{\varepsilon^2}f(\delta))$, then w.p. at least $1-\delta$ it holds that
	\begin{equation*}
	\frac{1}{\sqrt{1+\varepsilon}}\|\bm{z}_i - \bm{z}_j\|_2 \leq \|\bm{a}_i - \bm{a}_j\|_2 \leq \frac{1}{\sqrt{1-\varepsilon}}\|\bm{z}_i - \bm{z}_j\|_2.
	\end{equation*}}
\end{proposition}
\begin{proof}
By definition, we have $\bm{x}_i = \mathbf{X}\bm{z}_i = \mathbf{X}\mathbf{R}\bm{a}_i$, and thus
\begin{equation}
\label{eq:dist_no_norm}
 \mathbf{X}(\bm{z}_i - \bm{z}_j) = \mathbf{X}\mathbf{R}(\bm{a}_i - \bm{a}_j) = \bm{x}_i - \bm{x}_j .
\end{equation}
Let $\mathbf{X} = \mathbf{U}_\rho\mathbf{\Sigma}_{\rho}\mathbf{V}_\rho^{\top}$, and rewrite \eqref{eq:dist_no_norm} as
\begin{alignat}{1}
\label{eq:dist_bothsides}
 \mathbf{U}_\rho\mathbf{\Sigma}_\rho\mathbf{V}_\rho^{\top}(\bm{z}_i - \bm{z}_j) = 
 \mathbf{U}_\rho\mathbf{\Sigma}_\rho\mathbf{V}_\rho^{\top}\mathbf{R}(\bm{a}_i - \bm{a}_j). 
\end{alignat}
Left-multiplying by $\mathbf{\Sigma}_\rho^{-1}\mathbf{U}_\rho^{\top}$ reduces \eqref{eq:dist_bothsides} to
\begin{equation}
\label{eq:dist_bothsides2}
\mathbf{V}_\rho^{\top}(\bm{z}_i - \bm{z}_j) = \mathbf{V}_\rho^{\top}\mathbf{R}(\bm{a}_i - \bm{a}_j).
\end{equation}
Taking the norm of both sides, and noting that $\mathbf{V}$ is an orthonormal matrix implies that
\begin{equation}
\|\bm{z}_i - \bm{z}_j\|_2 = \|\mathbf{R}(\bm{a}_i - \bm{a}_j)\|_2
\end{equation}
which upon recalling Def.~\ref{def:JLT} yields
\begin{equation}
\begin{aligned}
& \|\bm{z}_i - \bm{z}_j\|_2\leq \sqrt{1 + \varepsilon}\|\bm{a}_i - \bm{a}_j\|_2, \\
& \sqrt{1 - \varepsilon}\|\bm{a}_i - \bm{a}_j\|_2\leq \|\bm{z}_i - \bm{z}_j\|_2 
\end{aligned}
\end{equation}
w.p. at least $1-\delta$.
\end{proof}

{
\section{Algorithm details}
\label{app:alg}
\subsection{ADMM algorithm for~\eqref{eq:sketched_SSC}}
Consider the Sketch-SSC for a single datum $\bm{x}$
\begin{equation}
\label{eq:sketch_ssc_admm1}
\min_{\bm{a}}~ \frac{\lambda}{2}\|\bm{x} - \mathbf{B}\bm{a}\|_2^2 + \|\bm{a}\|_1
\end{equation}
The optimization problem of~\eqref{eq:sketch_ssc_admm1} will be solved using the alternating direction method of multipliers~\cite{ADMMGG}. Define a new $n\times 1$ vector of auxiliary variables $\bm{c}$, and consider the following optimization problem that is equivalent to~\eqref{eq:sketch_ssc_admm1}
\begin{alignat}{2}
\label{eq:sketch_ssc_admm2}
&\min_{\bm{a},\bm{c}} &&~ \frac{\lambda}{2}\|\bm{x} - \mathbf{B}\bm{a}\|_2^2 + \|\bm{c}\|_1  \\
& \text{s. to.} &&~ \bm{a} = \bm{c}. \notag
\end{alignat}
The augmented Lagrangian of~\eqref{eq:sketch_ssc_admm2} is
\begin{equation}
\mathcal{L} = \frac{\lambda}{2}\|\bm{x} - \mathbf{B}\bm{a}\|_2^2 + \|\bm{c}\|_1 + \frac{\nu}{2}\|\bm{a} - \bm{c} + \bm{\delta}\|_2^2
\end{equation}
where $\bm{\delta}$ is a $n\times 1$ vector of dual variables and $\nu>0$ is a penalty parameter.
At each ADMM iteration the variables $\bm{a},\bm{c}$ are updated by setting the gradient of $\mathcal{L}$ w.r.t. $\bm{a}$ and $\bm{c}$ respectively to $\bm{0}$. Furthermore, the dual variables $\bm{\delta}$ are updated using a gradient ascent step at each iteration. The update of $\bm{a}$ at the $i$-th iteration is given by
\begin{alignat}{2}
& \frac{\partial\mathcal{L}}{\partial\bm{a}} = -\lambda\mathbf{B}^{\top}(\bm{x} - \mathbf{B}\bm{a}) + \nu(\bm{a} - \bm{c} + \bm{\delta}) = \bm{0} \Rightarrow \notag\\
& \bm{a}[i+1] = (\lambda\mathbf{B}^{\top}\mathbf{B} + \nu\mathbf{I})^{-1}(\lambda\mathbf{B}^{\top}\bm{x} + \nu(\bm{c}[i] - \bm{\delta}[i])) 
\label{eq:sk_ssc_admm_a}
\end{alignat}
where brackets indicate ADMM iteration indices. Accordingly, the update for $\bm{c}$ is given by
\begin{equation}
\bm{c}[i+1] = \mathcal{T}_{1/\nu}(\bm{a}[i+1] + \bm{\delta}[i])
\label{eq:sk_ssc_admm_c}
\end{equation}
where $\mathcal{T}_\sigma(\cdot)$ denotes the element-wise soft-thresholding operator
\begin{equation}
\mathcal{T}_\sigma(z) := \begin{cases}
z - \sigma \quad \text{ if } z > \sigma \\
0 ~~\quad\quad \text { if } |z|\leq \sigma \\
z + \sigma \quad \text{ if } z < -\sigma
\end{cases}.
\end{equation}
Finally, $\bm{\delta}$ is updated as
\begin{equation}
\bm{\delta}[i+1] = \bm{\delta}[i] + \bm{a}[i+1] - \bm{c}[i+1].
\label{eq:sk_ssc_admm_delta}
\end{equation}
The entire process is listed in Alg.~\ref{alg:admm_ssc}.
\begin{algorithm}[tb]
	\begin{algorithmic}[1]
		\algrenewcommand\algorithmicindent{1em}
		\Require{$D\times N$ data matrix $\mathbf{X}$; $D\times n$ basis $\mathbf{B}$; regularization parameter $\lambda$;}
		\Ensure{Model matrix $\mathbf{A}$;}
		\For {Each datum $\bm{x}_j$ to $\bm{x}_N$}
		\State Initialize $\bm{a}_j[0],\bm{c}[0],\bm{\delta}[0]$
		\Repeat
		\State Compute $\bm{a}_j[i+1]$ using \eqref{eq:sk_ssc_admm_a}
		\State Compute $\bm{c}[i+1]$ using \eqref{eq:sk_ssc_admm_c}
		\State Compute $\bm{\delta}[i+1]$ using \eqref{eq:sk_ssc_admm_delta}
		\State Update iteration counter $i\leftarrow i+1$
		\Until convergence
		\EndFor
		\State $\mathbf{A} = [\bm{a}_1,\ldots,\bm{a}_N].$
	\end{algorithmic}
	\caption{ADMM solver of Sketch-SSC [cf.~\eqref{eq:sketched_SSC}]}\label{alg:admm_ssc}
\end{algorithm}
\subsection{ALM algorithm for~\eqref{eq:sketched_LRR}}
Consider the Sketch-LRR 
\begin{equation}
\label{eq:sketch_lrr_alm1}
\min_{\mathbf{A}}~ \frac{\lambda}{2}\|\mathbf{X} - \mathbf{B}\mathbf{A}\|_F^2 + \|\mathbf{A}\|_*
\end{equation}
The optimization problem of~\eqref{eq:sketch_lrr_alm1} will be solved using the augmented Lagrangian method (ALM)~\cite{lin2010augmented,LRR}.  Define a new $n\times N$ matrix of auxiliary variables $\mathbf{C}$, and consider the following optimization task that is equivalent to~\eqref{eq:sketch_lrr_alm1}
\begin{alignat}{2}
\label{eq:sketch_lrr_alm2}
&\min_{\mathbf{A},\mathbf{C}} &&~ \frac{\lambda}{2}\|\mathbf{X} - \mathbf{B}\mathbf{A}\|_F^2 + \|\mathbf{C}\|_*  \\
& \text{s. to.} &&~ \mathbf{A} = \mathbf{C} \notag
\end{alignat}
The augmented Lagrangian of~\eqref{eq:sketch_lrr_alm2} is
\begin{equation}
\mathcal{L} = \frac{\lambda}{2}\|\mathbf{X} - \mathbf{B}\mathbf{A}\|_F^2 + \|\mathbf{C}\|_* + \frac{\nu}{2}\|\mathbf{A} - \mathbf{C} + \mathbf{\Delta}\|_F^2
\end{equation}
where $\mathbf{\Delta}$ is a $n\times N$ matrix of dual variables and $\nu>0$ is a penalty parameter. At each ALM iteration the variables $\mathbf{A},\mathbf{C}$ are updated by setting the gradient of $\mathcal{L}$ w.r.t. $\mathbf{A}$ and $\mathbf{C}$ respectively to $\bm{0}$. Furthermore, the dual variables $\mathbf{\Delta}$ are updated using a gradient ascent step per iteration. The update of $\mathbf{A}$ at the $i$-th iteration is given by
\begin{alignat}{2}
\label{eq:lrr_alm_a}
& \frac{\partial\mathcal{L}}{\partial\mathbf{A}} = \bm{0} \Rightarrow \\ & \mathbf{A}[i+1] = (\lambda\mathbf{B}^{\top}\mathbf{B} + \nu\mathbf{I})^{-1}(\lambda\mathbf{B}^{\top}\mathbf{X} - \nu(\mathbf{C}[i] - \mathbf{\Delta}[i])) \notag
\end{alignat}
where brackets indicate ALM iteration indices. Accordingly, the update for $\mathbf{C}$ is given by
\begin{equation}
\label{eq:alm_lrr_c_upd}
\mathbf{C}[i+1] = \arg\min_{\mathbf{C}} \frac{1}{\nu}\|\mathbf{C}\|_* + \frac{1}{2}\|\mathbf{C} - (\mathbf{A}[i+1] + \mathbf{\Delta}[i])\|_F^2.
\end{equation}
Note that the update \eqref{eq:alm_lrr_c_upd} can be performed using the Singular Value Thresholding algorithm~\cite{cai2010singular}. Finally $\mathbf{\Delta}$ is updated as
\begin{equation}
\label{eq:alm_lrr_delta}
\mathbf{\Delta}[i+1] = \mathbf{\Delta}[i] + \mathbf{A}[i+1] - \mathbf{C}[i+1]
\end{equation}
and the penalty parameter is also updated as
\begin{equation}
\label{eq:alm_lrr_nu}
\nu = \min(p\nu,\nu_{\rm max})
\end{equation}
where $p>1$ is a prescribed constant, and $\nu_{\rm max}$ is a predefined maximum limit for $\nu$.

\begin{algorithm}[tb]
	\begin{algorithmic}[1]
		\algrenewcommand\algorithmicindent{1em}
		\Require{$D\times N$ data matrix $\mathbf{X}$; $D\times n$ basis $\mathbf{B}$; regularization parameter $\lambda$;}
		\Ensure{Model matrix $\mathbf{A}$;}
		\State Initialize $\mathbf{A},\mathbf{C},\mathbf{\Delta}$
		\Repeat
		\State Compute $\mathbf{A}[i+1]$ using \eqref{eq:lrr_alm_a}
		\State Compute $\mathbf{C}[i+1]$ using \eqref{eq:alm_lrr_c_upd}
		\State Compute $\bm{\delta}[i+1]$ using \eqref{eq:alm_lrr_delta}
		\State Update $\nu$ using \eqref{eq:alm_lrr_nu}
		\State Update iteration counter $i\leftarrow i+1$
		\Until convergence
	\end{algorithmic}
	\caption{ALM solver of Sketch-LRR [cf.~\eqref{eq:sketched_LRR}]}\label{alg:alm_lrr}
\end{algorithm}
}
\ifCLASSOPTIONcaptionsoff
\newpage
\fi



\bibliographystyle{IEEEtran}
\bibliography{./bib/SketchedSC}

\begin{thebibliography}{10}
\providecommand{\url}[1]{#1}
\csname url@samestyle\endcsname
\providecommand{\newblock}{\relax}
\providecommand{\bibinfo}[2]{#2}
\providecommand{\BIBentrySTDinterwordspacing}{\spaceskip=0pt\relax}
\providecommand{\BIBentryALTinterwordstretchfactor}{4}
\providecommand{\BIBentryALTinterwordspacing}{\spaceskip=\fontdimen2\font plus
\BIBentryALTinterwordstretchfactor\fontdimen3\font minus
  \fontdimen4\font\relax}
\providecommand{\BIBforeignlanguage}[2]{{%
\expandafter\ifx\csname l@#1\endcsname\relax
\typeout{** WARNING: IEEEtran.bst: No hyphenation pattern has been}%
\typeout{** loaded for the language `#1'. Using the pattern for}%
\typeout{** the default language instead.}%
\else
\language=\csname l@#1\endcsname
\fi
#2}}
\providecommand{\BIBdecl}{\relax}
\BIBdecl

\bibitem{hastie01statisticallearning}
T.~Hastie, R.~Tibshirani, and J.~Friedman, \emph{The Elements of Statistical
  Learning}.\hskip 1em plus 0.5em minus 0.4em\relax New~York: Springer, 2001.

\bibitem{vidal2010tutorial}
R.~Vidal, ``A tutorial on subspace clustering,'' \emph{IEEE Signal Process.\
  Magazine}, vol.~28, no.~2, pp. 52--68, 2010.

\bibitem{woodruff}
D.~P. Woodruff, ``Sketching as a tool for numerical linear algebra,''
  \emph{Foundations and Trends in Theoretical Computer Science}, vol.~10, no.
  1--2, pp. 1--157, 2014.

\bibitem{boutsidis}
C.~Boutsidis, A.~Zouzias, M.~W. Mahoney, and P.~Drineas, ``Randomized
  dimensionality reduction for k-means clustering,'' \emph{IEEE Transactions on
  Information Theory}, vol.~61, no.~2, pp. 1045--1062, 2015.

\bibitem{JL}
W.~B. Johnson and J.~Lindenstrauss, ``Extensions of {L}ipschitz mappings into a
  {H}ilbert space,'' \emph{Contemporary {M}athematics}, vol.~26, no. 189-206,
  p.~1, 1984.

\bibitem{VidalOMP}
C.~You, D.~Robinson, and R.~Vidal, ``Scalable sparse subspace clustering by
  orthogonal matching pursuit,'' in \emph{IEEE Conference on Computer Vision
  and Pattern Recognition}, vol.~1, 2016.

\bibitem{ORGEN}
C.~You, C.-G. Li, D.~P. Robinson, and R.~Vidal, ``Oracle based active set
  algorithm for scalable elastic net subspace clustering,'' in \emph{Proc. of
  IEEE Conf. on Computer Vision and Pattern Recognition}, Las Vegas, NV, June
  2016.

\bibitem{traganitis2016sketchedSC}
P.~A. Traganitis and G.~B. Giannakis, ``A randomized approach to large-scale
  subspace clustering,'' in \emph{50th Asilomar Conference on Signals, Systems
  and Computers}.\hskip 1em plus 0.5em minus 0.4em\relax IEEE, 2016, pp.
  1019--1023.

\bibitem{jolliffe2002principal}
I.~Jolliffe, \emph{Principal Component Analysis}.\hskip 1em plus 0.5em minus
  0.4em\relax Wiley Online Library, 2002.

\bibitem{parsons2004subspace}
L.~Parsons, E.~Haque, and H.~Liu, ``Subspace clustering for high dimensional
  data: {A} review,'' \emph{ACM SIGKDD Explorations Newsletter}, vol.~6, no.~1,
  pp. 90--105, 2004.

\bibitem{lloydkmeans}
S.~Lloyd, ``Least-squares quantization in {PCM},'' \emph{IEEE Trans.\ Info.\
  Theory}, vol.~28, no.~2, pp. 129--137, 1982.

\bibitem{ksubspaces}
P.~K. Agarwal and N.~H. Mustafa, ``{$K$}-means projective clustering,'' in
  \emph{Proc.\ 23rd ACM SIGMOD-SIGACT-SIGART Symposium}.\hskip 1em plus 0.5em
  minus 0.4em\relax Paris, France: ACM, June 2004, pp. 155--165.

\bibitem{tipping1999mixtures}
M.~Tipping and C.~Bishop, ``Mixtures of probabilistic principal component
  analyzers,'' \emph{Neural Computation}, vol.~11, no.~2, pp. 443--482, 1999.

\bibitem{ALC2007}
Y.~Ma, H.~Derksen, W.~Hong, and J.~Wright, ``Segmentation of multivariate mixed
  data via lossy data coding and compression,'' \emph{IEEE Trans.\ Pattern
  Analysis Machine Intelligence}, vol.~29, no.~9, pp. 1546--1562, 2007.

\bibitem{gpca}
R.~Vidal, Y.~Ma, and S.~Sastry, ``Generalized principal component analysis
  {(GPCA)},'' \emph{IEEE Trans.\ Pattern Analysis Machine Intelligence},
  vol.~27, no.~12, pp. 1945--1959, 2005.

\bibitem{zhang2012hybrid}
T.~Zhang, A.~Szlam, Y.~Wang, and G.~Lerman, ``Hybrid linear modeling via local
  best-fit flats,'' \emph{Intern.\ J.\ Computer Vision}, vol. 100, no.~3, pp.
  217--240, 2012.

\bibitem{heckel2015robust}
R.~Heckel and H.~B{\"o}lcskei, ``Robust subspace clustering via thresholding,''
  \emph{IEEE Transactions on Information Theory}, vol.~61, no.~11, pp.
  6320--6342, 2015.

\bibitem{rahmani2015innovation}
M.~Rahmani and G.~Atia, ``Innovation pursuit: A new approach to subspace
  clustering,'' \emph{arXiv preprint arXiv:1512.00907}, 2015.

\bibitem{traganitis2016parafac}
P.~A. Traganitis and G.~B. Giannakis, ``{PARAFAC}-based multilinear subspace
  clustering for tensor data,'' in \emph{IEEE Global Conference on Signal and
  Information Processing}.\hskip 1em plus 0.5em minus 0.4em\relax Washington
  DC: IEEE, 2016, pp. 1280--1284.

\bibitem{zhang2009median}
T.~Zhang, A.~Szlam, and G.~Lerman, ``Median $k$-flats for hybrid linear
  modeling with many outliers,'' in \emph{Proc.\ of ICCV}.\hskip 1em plus 0.5em
  minus 0.4em\relax Kyoto, Japan: IEEE, September 2009, pp. 234--241.

\bibitem{traganitis2016online}
P.~A. Traganitis and G.~B. Giannakis, ``Efficient subspace clustering of
  large-scale data streams with misses,'' in \emph{Annual Conference on
  Information Science and Systems}.\hskip 1em plus 0.5em minus 0.4em\relax
  Princeton, NJ: IEEE, 2016, pp. 590--595.

\bibitem{onlineLRR}
J.~Shen, P.~Li, and H.~Xu, ``Online low-rank subspace clustering by basis
  dictionary pursuit,'' in \emph{Proceedings of The 33rd International
  Conference on Machine Learning}, New York, NY, 2016, pp. 622--631.

\bibitem{spectralclustering}
U.~Von~Luxburg, ``A tutorial on spectral clustering,'' \emph{Statistics and
  Computing}, vol.~17, no.~4, pp. 395--416, 2007.

\bibitem{elhamifar2013SSC}
E.~Elhamifar and R.~Vidal, ``Sparse subspace clustering: {A}lgorithm, theory,
  and applications,'' \emph{IEEE Trans.\ Pattern Analysis Machine
  Intelligence}, vol.~35, no.~11, pp. 2765--2781, 2013.

\bibitem{LRR}
G.~Liu, Z.~Lin, and Y.~Yu, ``Robust subspace segmentation by low-rank
  representation,'' in \emph{Proc.\ ICML}, Haifa, Israel, June 2010, pp.
  663--670.

\bibitem{LSR}
C.~Lu, H.~Min, Z.~Zhao, L.~Zhu, D.~Huang, and S.~Yan, ``Robust and efficient
  subspace segmentation via least-squares regression,'' in \emph{European
  {C}onference on {C}omputer {V}ision}.\hskip 1em plus 0.5em minus 0.4em\relax
  Florence, Italy: Springer, 2012, pp. 347--360.

\bibitem{panagakis2014elastic}
Y.~Panagakis and C.~Kotropoulos, ``Elastic net subspace clustering applied to
  pop/rock music structure analysis,'' \emph{Pattern Recognition Letters},
  vol.~38, pp. 46--53, 2014.

\bibitem{fang2015ensc}
Y.~Fang, R.~Wang, B.~Dai, and X.~Wu, ``Graph-based learning via auto-grouped
  sparse regularization and kernelized extension,'' \emph{IEEE Transactions on
  Knowledge and Data Engineering}, vol.~27, no.~1, pp. 142--154, 2015.

\bibitem{heckel2017dimensionality}
R.~Heckel, M.~Tschannen, and H.~B{\"o}lcskei, ``Dimensionality-reduced subspace
  clustering,'' \emph{arXiv preprint arXiv:1507.07105}, 2017.

\bibitem{pimentel2016necessary}
D.~Pimentel-Alarc{\'o}n, L.~Balzano, and R.~Nowak, ``Necessary and sufficient
  conditions for sketched subspace clustering,'' in \emph{54th Annual Allerton
  Conference on Communication, Control, and Computing}.\hskip 1em plus 0.5em
  minus 0.4em\relax Champaign, IL: IEEE, 2016, pp. 1335--1343.

\bibitem{wang2016theoretical}
Y.~Wang, Y.-X. Wang, and A.~Singh, ``A theoretical analysis of noisy sparse
  subspace clustering on dimensionality-reduced data,'' \emph{arXiv preprint
  arXiv:1610.07650}, 2016.

\bibitem{pourkamali2017preconditioned}
F.~Pourkamali-Anaraki and S.~Becker, ``Preconditioned data sparsification for
  big data with applications to {PCA} and {K-means},'' \emph{IEEE Transactions
  on Information Theory}, vol.~63, no.~5, pp. 2954--2974, 2017.

\bibitem{skeva}
P.~A. Traganitis, K.~Slavakis, and G.~B. Giannakis, ``Sketch and validate for
  big data clustering,'' \emph{IEEE J.\ Selected Topics Signal Processing},
  vol.~9, no.~4, pp. 678--690, June 2015.

\bibitem{SSSC2}
X.~Peng, H.~Tang, L.~Zhang, Z.~Yi, and S.~Xiao, ``A unified framework for
  representation-based subspace clustering of out-of-sample and large-scale
  data.'' \emph{IEEE {T}rans. on {N}eural {N}etworks and {L}earning {S}ystems},
  vol.~27, no.~12, pp. 2499--2512, 2016.

\bibitem{dyerOMP}
E.~L. Dyer, A.~C. Sankaranarayanan, and R.~G. Baraniuk, ``Greedy feature
  selection for subspace clustering.'' \emph{Journal of Machine Learning
  Research}, vol.~14, no.~1, pp. 2487--2517, 2013.

\bibitem{achlioptas}
D.~Achlioptas, ``Database-friendly random projections: Johnson-{L}indenstrauss
  with binary coins,'' \emph{Journal of {C}omputer and System Sciences},
  vol.~66, no.~4, pp. 671--687, 2003.

\bibitem{Mailman}
E.~Liberty and S.~W. Zucker, ``The mailman algorithm: A note on matrix-vector
  multiplication,'' \emph{Information Processing Letters}, vol. 109, no.~3, pp.
  179--182, 2009.

\bibitem{ailon2009fast}
N.~Ailon and B.~Chazelle, ``The fast {Johnson--Lindenstrauss} transform and
  approximate nearest neighbors,'' \emph{SIAM Journal on Computing}, vol.~39,
  no.~1, pp. 302--322, 2009.

\bibitem{ailon2009fast2}
N.~Ailon and E.~Liberty, ``Fast dimension reduction using {R}ademacher series
  on dual {BCH} codes,'' \emph{Discrete \& Computational Geometry}, vol.~42,
  no.~4, p. 615, 2009.

\bibitem{anaraki2014memory}
F.~Pourkamali-Anaraki and S.~Hughes, ``Memory and computation efficient pca via
  very sparse random projections,'' in \emph{Proceedings of the 31st
  International Conference on Machine Learning (ICML-14)}, 2014, pp.
  1341--1349.

\bibitem{clarkson2013low}
K.~L. Clarkson and D.~P. Woodruff, ``Low rank approximation and regression in
  input sparsity time,'' in \emph{Proceedings of the forty-fifth annual ACM
  symposium on Theory of computing}.\hskip 1em plus 0.5em minus 0.4em\relax
  ACM, 2013, pp. 81--90.

\bibitem{ADMMGG}
G.~B. Giannakis, Q.~Ling, G.~Mateos, I.~D. Schizas, and H.~Zhu, ``Decentralized
  learning for wireless communications and networking,'' in \emph{Splitting
  Methods in Communication and Imaging, Science and Engineering}, R.~Glowinski,
  S.~Osher, and W.~Yin, Eds.\hskip 1em plus 0.5em minus 0.4em\relax Springer,
  2016.

\bibitem{OnlineLRRTraganitis}
P.~A. Traganitis and G.~B. Giannakis, ``Efficient subspace clustering of
  large-scale data streams with misses,'' in \emph{Annual Conference on
  Information Science and Systems}.\hskip 1em plus 0.5em minus 0.4em\relax
  Princeton, NJ: IEEE, March 2016.

\bibitem{fastfood}
\BIBentryALTinterwordspacing
Q.~Le, T.~Sarlos, and A.~Smola, ``Fastfood - approximating kernel expansions in
  loglinear time,'' in \emph{30th International Conference on Machine
  Learning}, Atlanta, GA, 2013. [Online]. Available:
  \url{http://jmlr.org/proceedings/papers/v28/le13.html}
\BIBentrySTDinterwordspacing

\bibitem{lehoucq1998arpack}
R.~B. Lehoucq, D.~C. Sorensen, and C.~Yang, ``{ARPACK} users' guide: solution
  of large-scale eigenvalue problems with implicitly restarted arnoldi
  methods,'' vol.~6.\hskip 1em plus 0.5em minus 0.4em\relax Soc. for Industrial
  and Applied Math, 1998.

\bibitem{kalantzis2016spectral}
V.~Kalantzis, R.~Li, and Y.~Saad, ``Spectral {S}chur complement techniques for
  symmetric eigenvalue problems,'' \emph{Electronic Transactions on Numerical
  Analysis}, vol.~45, pp. 305--329, 2016.

\bibitem{l2knng}
D.~C. Anastasiu and G.~Karypis, ``L2knng: Fast exact k-nearest neighbor graph
  construction with l2-norm pruning,'' in \emph{Proceedings of the 24th ACM
  International Conference on Information and Knowledge Management}.\hskip 1em
  plus 0.5em minus 0.4em\relax Melbourne, Australia: ACM, 2015, pp. 791--800.

\bibitem{greedyfiltering}
Y.~Park, S.~Park, S.-g. Lee, and W.~Jung, ``Greedy filtering: A scalable
  algorithm for k-nearest neighbor graph construction,'' in \emph{International
  Conference on Database Systems for Advanced Applications}.\hskip 1em plus
  0.5em minus 0.4em\relax Bali, Indonesia: Springer, 2014, pp. 327--341.

\bibitem{ITQ}
Y.~Gong, S.~Lazebnik, A.~Gordo, and F.~Perronnin, ``Iterative quantization: A
  {P}rocrustean approach to learning binary codes for large-scale image
  retrieval,'' \emph{IEEE Transactions on Pattern Analysis and Machine
  Intelligence}, vol.~35, no.~12, pp. 2916--2929, 2013.

\bibitem{indyk1998approximate}
P.~Indyk and R.~Motwani, ``Approximate nearest neighbors: {T}owards removing
  the curse of dimensionality,'' in \emph{Proceedings of the 30th Annual ACM
  {S}ymposium on Theory of {C}omputing}.\hskip 1em plus 0.5em minus 0.4em\relax
  Dallas, TX: ACM, 1998, pp. 604--613.

\bibitem{slaney2008locality}
M.~Slaney and M.~Casey, ``Locality-sensitive hashing for finding nearest
  neighbors [lecture notes],'' \emph{IEEE Signal Processing Magazine}, vol.~25,
  no.~2, pp. 128--131, 2008.

\bibitem{MATLAB:2015}
MATLAB, \emph{version 8.6.0 (R2015b)}.\hskip 1em plus 0.5em minus 0.4em\relax
  Natick, Massachusetts: The MathWorks Inc., 2015.

\bibitem{VLfeat}
A.~Vedaldi and B.~Fulkerson, ``{VLFeat}: An open and portable library of
  computer vision algorithms,'' \url{http://www.vlfeat.org/}, 2008.

\bibitem{hopkins155}
R.~Tron and R.~Vidal, ``A benchmark for the comparison of {3-D} motion
  segmentation algorithms,'' in \emph{IEEE Conference on Computer Vision and
  Pattern Recognition}, Minneapolis, MN, 2007, pp. 1--8.

\bibitem{yaleb}
A.~S. Georghiades, P.~N. Belhumeur, and D.~J. Kriegman, ``From few to many:
  Illumination cone models for face recognition under variable lighting and
  pose,'' \emph{IEEE Trans.\ Pattern Analysis Machine Intelligence}, vol.~23,
  no.~6, pp. 643--660, June 2001.

\bibitem{COIL100}
S.~A. Nene, S.~K. Nayar, and H.~Murase, ``Columbia object image library
  (coil-100),'' CUCS-006-96, Tech. Rep., 1996.

\bibitem{MNIST}
Y.~LeCun, L.~Bottou, Y.~Bengio, and P.~Haffner, ``Gradient-based learning
  applied to document recognition,'' \emph{Proceedings of the IEEE}, vol.~86,
  no.~11, pp. 2278--2324, 1998.

\bibitem{ScatNet}
J.~Bruna and S.~Mallat, ``Invariant scattering convolution networks,''
  \emph{IEEE Trans. on Pattern Analysis and Machine Intelligence}, vol.~35,
  no.~8, pp. 1872--1886, 2013.

\bibitem{sarlos}
T.~Sarlos, ``Improved approximation algorithms for large matrices via random
  projections,'' in \emph{47th Annual IEEE Symposium on Foundations of Computer
  Science}, Berkeley, CA, 2006, pp. 143--152.

\bibitem{yang2015randomized}
Y.~Yang, M.~Pilanci, and M.~J. Wainwright, ``Randomized sketches for kernels:
  Fast and optimal non-parametric regression,'' \emph{arXiv preprint
  arXiv:1501.06195}, 2015.

\bibitem{lin2010augmented}
Z.~Lin, M.~Chen, and Y.~Ma, ``The augmented {L}agrange multiplier method for
  exact recovery of corrupted low-rank matrices,'' \emph{arXiv preprint
  arXiv:1009.5055}, 2010.

\bibitem{cai2010singular}
J.-F. Cai, E.~J. Cand{\`e}s, and Z.~Shen, ``A singular value thresholding
  algorithm for matrix completion,'' \emph{SIAM Journal on Optimization},
  vol.~20, no.~4, pp. 1956--1982, 2010.

\end{thebibliography}

\end{document}